\documentclass[preprint,11pt]{article}

\usepackage{amssymb,amsfonts,amsmath,amsthm,amscd,dsfont,mathrsfs}
\usepackage{graphicx,float,psfrag,epsfig}
\usepackage{wrapfig}

\DeclareMathAlphabet{\mathpzc}{OT1}{pzc}{m}{it}

\newtheorem{propo}{Proposition}[section]
\newtheorem{lemma}[propo]{Lemma}

\newtheorem{thm}[propo]{Theorem}

\newtheorem{remark}[propo]{Remark}

\def\tF{\widetilde{F}}
\def\bL{\overline{L}}
\def\cA{{\cal A}}
\def\cH{{\cal H}}
\def\cC{{\cal C}}
\def\cG{{\cal G}}
\def\cE{{\cal E}}
\def\T{{\sf T}}
\def\eps{\epsilon}

\def\tB{\widetilde{B}}

\def\bL{\overline{L}}
\def\reals{{\mathds R}}
\def\tM{\widetilde{M}}

\def\tMij{\widetilde{M}^E_{ij}}

\def\Mmax{M_{\rm max}}

\def\Smax{\Sigma_{\rm max}}
\def\Smin{\Sigma_{\rm min}}

\def\diag{{\rm diag}}
\def\spn{{\rm span}}
\def\prob{{\mathbb P}}
\def\P{{\mathbb P}}
\def\Z{\mathbb Z}
\def\E{\mathbb E}

\def\dim{{\rm dim}}

\def\ind{{\mathbb I}}

\def\<{\langle}
\def\>{\rangle}
\def\diag{{\rm diag}}
\def\Grass{{\sf G}}
\def\Orth{{\sf O}}
\def\Manif{{\sf M}}
\def\id{{\mathbf 1}}
\def\cP{{\cal P}}
\def\cF{{\cal F}}

\def\Xm{{\bf x}}
\def\Wm{{\bf w}}

\def\Um{{\bf u}}
\def\Whm{\widehat{\bf w}}
\def\eps{\epsilon}
\def\Tang{{\sf T}}
\def\Trace{{\rm Tr}}
\def\hM{\widehat{M}}
\def\hA{\widehat{A}}
\def\hB{\widehat{B}}
\def\hW{\widehat{W}}
\def\hZ{\widehat{Z}}
\def\oW{\overline{W}}
\def\oZ{\overline{Z}}
\def\E{{\mathbb E}}
\def\grad{{\rm grad}\, }
\def\Co{{\cal K}}

\begin{document}

\title{Matrix Completion from a Few Entries}

\author{Raghunandan~H.~Keshavan\thanks{Department of  Electrical Engineering, Stanford University},
        Andrea~Montanari${}^*$\thanks{Departments of Statistics, Stanford University},
        and~Sewoong~Oh${}^*$}

\maketitle
%
%
\begin{abstract}
Let $M$ be an $n\alpha\times n$ matrix of rank $r\ll n$, 
and assume that a uniformly random subset $E$ 
of its entries is observed. 
We describe an efficient algorithm that 
reconstructs $M$ from $|E| = O(r\,n)$ observed entries
with relative root mean square error 
\begin{eqnarray*}
{\rm RMSE} \le C(\alpha)\, \left(\frac{nr}{|E|}\right)^{1/2}\, .
\end{eqnarray*} 
Further, if  $r = O(1)$ and $M$ is sufficiently unstructured, then it can be reconstructed 
\emph{exactly}  from $|E| = O(n\log n)$ entries.

This settles (in the case of bounded rank)
a question left open by Cand{\`e}s and Recht  and improves
over the guarantees for their reconstruction algorithm.
The complexity of our algorithm is $O(|E|r\log n)$, which opens the way to
its use for massive  data sets. In the process of proving these statements,
we obtain a generalization of a celebrated result by
Friedman-Kahn-Szemer\'edi and Feige-Ofek on the spectrum of sparse random 
matrices.
\end{abstract}

%
%
\section{Introduction}

Imagine that each of $m$ customers watches and rates a subset 
of the $n$ movies available through a movie rental service.
This yields a dataset of customer-movie pairs $(i,j)\in 
E\subseteq [m]\times [n]$ and, for each such pair, a rating $M_{ij}\in \reals$.
The objective of \emph{collaborative filtering} is to predict 
the rating for the missing pairs in such a way as to provide targeted 
suggestions.\footnote{Indeed, in 2006, {\sc Netflix} made public such
a dataset with $m\approx 5\cdot 10^5$, $n\approx 2\cdot 10^4$ and
$|E|\approx 10^{8}$ and challenged the research community
to predict the missing ratings with root mean square error below
$0.8563$ \cite{Net06}.} The general question we address here is:
Under which conditions do the known ratings provide sufficient
information to infer the unknown ones?
Can this inference problem be solved efficiently?
The second question is particularly important in view of the 
massive size of actual data sets.
%
%
\subsection{Model definition}\label{sec:model}

A simple mathematical model for such 
data assumes that the (unknown) matrix of ratings has rank $r\ll m,n$.
More precisely, we denote by $M$ the matrix whose entry $(i,j)\in [m]\times
[n]$ corresponds to the rating user $i$ would assign to movie $j$.
We assume that there exist matrices $U$, of dimensions  $m \times r$, and 
$V$, of dimensions $n\times r$, and a diagonal matrix $\Sigma$, 
of dimensions $r\times r$ such that
\begin{eqnarray}
  M = U\Sigma V^T\, .\label{eq:MatrixForm}
\end{eqnarray}
For justification of these assumptions and background on the use of low rank
matrices in information retrieval, we refer to 
\cite{BDJ99}.
Since we are interested in very large data sets, we shall focus on the limit 
$m, n\to \infty$ with $m/n=  \alpha$ bounded away from $0$ and $\infty$.

We further assume that the factors $U$, $V$ are unstructured.
This notion is formalized by the \emph{incoherence condition} 
introduced by Cand\'es and Recht \cite{CaR08}, 
and defined in Section \ref{sec:coherence}. In particular the incoherence
condition is satisfied with high probability if  $M= U\Sigma V^T$ with 
$U$ and $V$ uniformly random matrices with $U^TU=m\id$ and $V^TV=n\id$.
Alternatively, incoherence holds 
if the entries of $U$ and $V$ are i.i.d. bounded
random variables.

Out of the $m\times n$ entries of $M$, a subset 
$E\subseteq[m]\times [n]$ (the user/movie pairs for which
a rating is available) is revealed.
We let  $M^E$ be the $m\times n$ matrix that contains the
revealed entries of $M$, and is
filled with $0$'s in the other positions
 \begin{eqnarray}
  M^E_{i,j} = \left\{
              \begin{array}{rl}
              M_{i,j} & \text{if } (i,j)\in E\, ,\\
              0       & \text{otherwise.}
              \end{array} \right. \label{eq:RevealedMatrixForm}
  \end{eqnarray}
The set $E$ will be uniformly  random given its 
size $|E|$.

%
%
\subsection{Algorithm}\label{sec:algorithm}

A naive algorithm consists of the following projection operation.

\vspace{0.05cm}

\noindent{\bf Projection.} Compute the singular value decomposition (SVD)
of $M^E$ (with $\sigma_1\ge\sigma_2\ge\cdots\ge 0$) 
\begin{equation}
M^E = \sum_{i=1}^{\min(m,n)}\sigma_i  x_iy_i^T\, ,\label{eq:WrongSVD}
\end{equation}
and return the matrix $\T_r(M^E) =  
(mn/|E|)\sum_{i=1}^{r}\sigma_i  x_iy_i^T$
obtained by setting to $0$ all but the $r$ largest singular values.
Notice that, apart from the rescaling factor
$(mn/|E|)$, $\T_{r}(M^E)$ is the orthogonal projection of $M^E$
onto the set of rank-$r$ matrices. The rescaling factor compensates  
the smaller average size of the entries of $M^E$ with respect to $M$.

\vspace{0.1cm}

It turns out that, if $|E|=\Theta(n)$, this algorithm performs
very poorly. The reason is that the matrix $M^E$ 
contains columns and rows with $\Theta(\log n/\log\log n)$ non-zero
(revealed) entries. 
The largest singular values of $M^E$ 
are of order $\Theta(\sqrt{\log n/\log\log n})$. The corresponding
singular vectors are highly concentrated on high-weight
column or row indices (respectively, for left and right singular vectors).
Such singular vectors are 
an artifact of the high-weight columns/rows and do not provide 
useful information about the hidden entries of $M$. 
This motivates the definition of
the following operation (hereafter the \emph{degree} of a
column or of a row is the number of its revealed entries).

\vspace{0.1cm}
\begin{figure}
\includegraphics[width=8.cm]{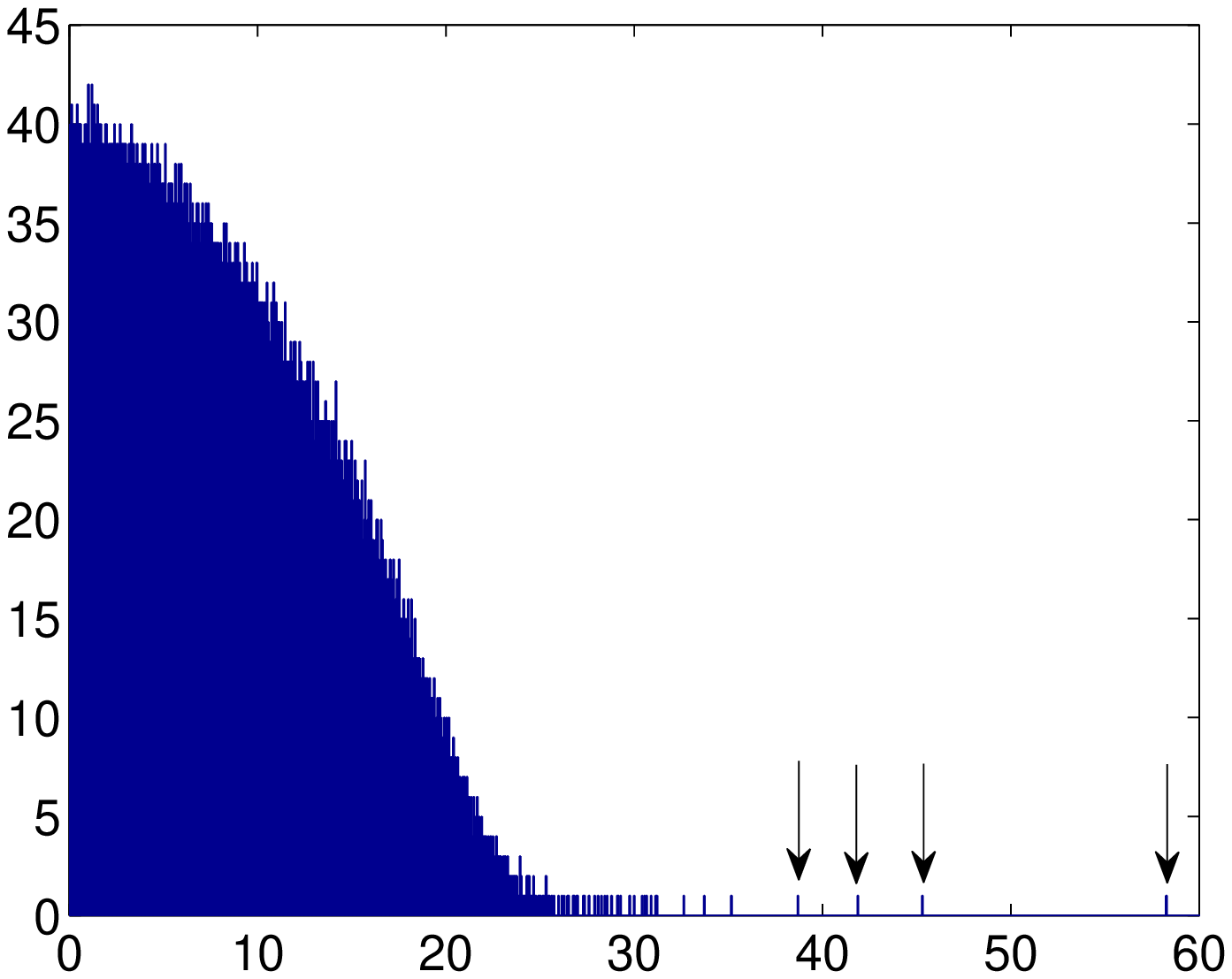}
\includegraphics[width=8.cm]{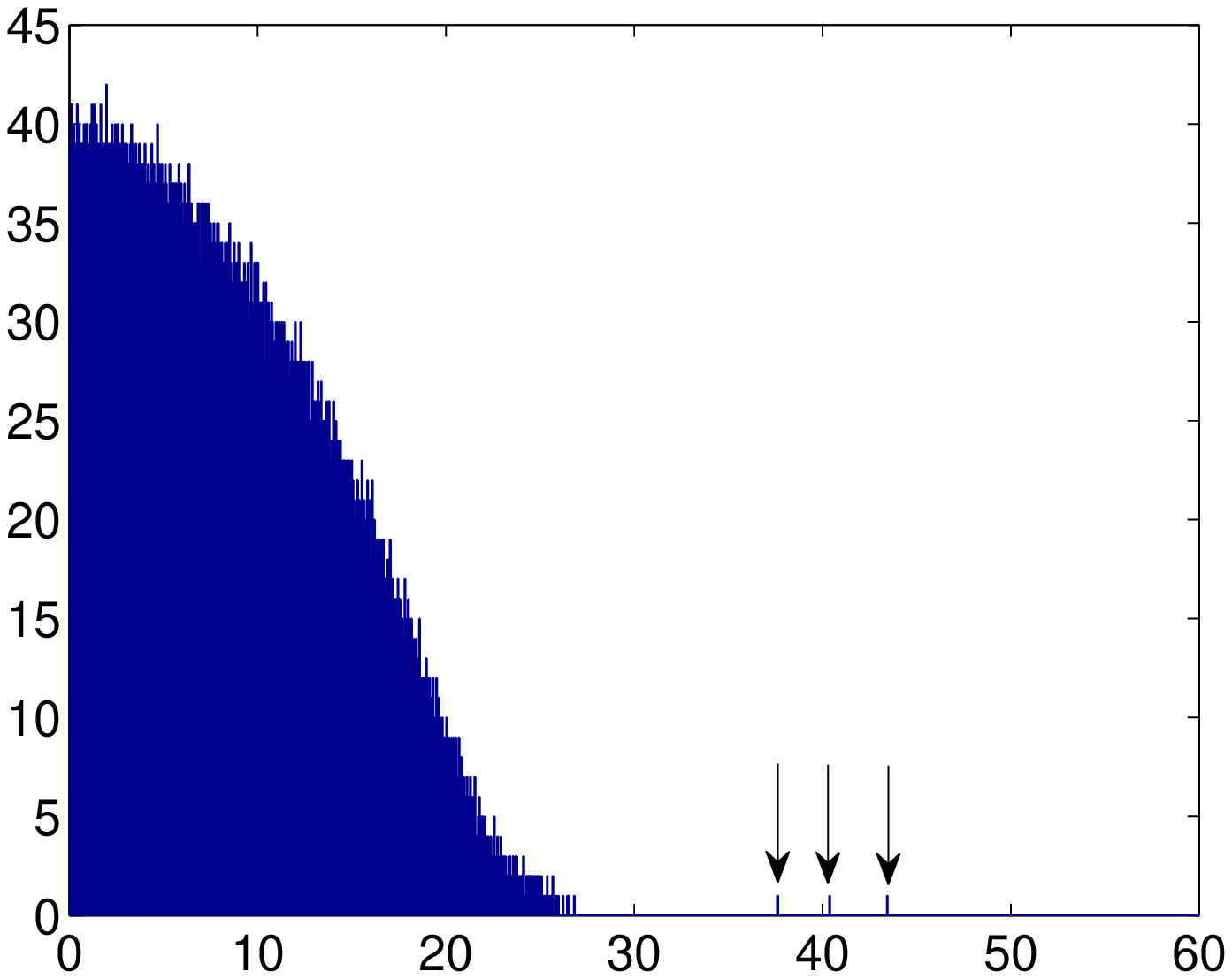}
\put(-321,46){\scriptsize{$\sigma_4$}}
\put(-310,46){\scriptsize{$\sigma_3$}}
\put(-298,46){\scriptsize{$\sigma_2$}}
\put(-263,46){\scriptsize{$\sigma_1$}}
\put(-92,46){\scriptsize{$\sigma_3$}}
\put(-81,46){\scriptsize{$\sigma_2$}}
\put(-70,46){\scriptsize{$\sigma_1$}}
\caption{{\small 
Histogram of the singular values of a partially 
revealed matrix $M^E$ before trimming (left) 
and after trimming (right) for $10^4 \times 10^4$ random rank-$3$ 
matrix $M$ with $\eps=30$ and $\Sigma=\diag(1,1.1,1.2)$.
After trimming the underlying rank-$3$ structure becomes clear.
Here the number of revealed entries per row follows a heavy tail distribution
with $\prob\{N=k\}= {\rm const.}/k^3$.}}
\label{fig:singularvalues}
\end{figure}

\noindent{\bf Trimming.} Set to zero all columns in $M^E$
with degree larger that $2|E|/n$. Set to $0$ all rows with degree
larger than $2|E|/m$. 

Figure \ref{fig:singularvalues} shows the singular value
distributions of $M^E$ and $\tM^E$ for a random rank-$3$ matrix $M$.
The surprise is that trimming (which amounts to `throwing out information')
makes the underlying rank-$3$ structure much more apparent.
This effect becomes even more important when the number of revealed entries 
per row/column follows a heavy tail distribution, as
for real data.

In terms of the above routines, our algorithm has the following structure.

\vspace{0.3cm}

\begin{tabular}{ll}
\hline
\multicolumn{2}{l}{ {\sc Spectral Matrix Completion}( matrix $M^E$ )}\\
\hline
1: & Trim $M^E$, and let $\tM^E$ be the output;\\
2: & Project $\tM^E$ to $\T_r(\tM^E)$;\\
3: & Clean residual errors by minimizing the discrepancy 
$F(X,Y)$.\\
\hline
\end{tabular}

\vspace{0.3cm}

The last step of the above algorithm allows to reduce (or eliminate) small 
discrepancies between $\T_r(\tM^E)$ and $M$, and is described below.

\vspace{0.1cm}

\noindent{\bf Cleaning.} Various implementations are possible,
but we found the following one particularly appealing.
Given $X\in\reals^{m\times r}$, $Y\in\reals^{n\times r}$
with $X^TX = m\id$ and $Y^TY=n\id$, we define 
\begin{eqnarray}
F(X,Y) & \equiv &\min_{S\in \reals^{r\times r}}\cF(X,Y,S)\, ,\label{eq:MinimizeS}\\
\cF(X,Y,S)&\equiv &  \frac{1}{2}\sum_{(i,j)\in E}(M_{ij}-(XSY^T)_{ij})^2 \, .
\end{eqnarray}
The cleaning step consists in writing $\T_r(\tM^E) = X_0S_0Y_0^T$
and minimizing $F(X,Y)$ locally with initial condition
$X=X_0$, $Y=Y_0$.

Notice that $F(X,Y)$ is easy to evaluate since it is defined 
by minimizing the quadratic function $S\mapsto \cF(X,Y,S)$ over 
the low-dimensional matrix $S$. Further 
it depends on $X$ and $Y$ only through their column spaces.
In geometric terms, $F$ is a function defined over the 
cartesian product of two Grassmann manifolds 
(we refer to Section \ref{sec:result2} for background
and references). Optimization over
Grassmann manifolds is a well understood topic
\cite{Edelman} and efficient
algorithms (in particular Newton and conjugate gradient)
can be applied. To be definite, we assume that gradient descent 
with line search is used to minimize $F(X,Y)$.

Finally, the implementation proposed here implicitly assumes that the rank 
$r$ is known. In practice this is a non-issue. Since $r\ll n$,
a loop over the value of $r$ can be added at little extra cost.
For instance, in 
collaborative filtering applications, $r$ ranges between $10$ and $30$.
%
%
\subsection{Main results}\label{sec:MainResults}

Notice that computing $\T_r(\tM^E)$ only requires to 
find the first $r$ singular vectors of a sparse matrix.
Our main result establishes that 
this simple procedure  achieves arbitrarily small relative root mean square
error from $O(nr)$ revealed entries. 
We define the relative root mean square error as
\begin{eqnarray}
{\rm RMSE} \equiv 
\left[\frac{1}{mn\Mmax^2}||M-\T_r(\tM^E)||_{\rm F}^2\right]^{1/2}\, .
\end{eqnarray}
where we denote by $||A||_F$ the Frobenius norm of matrix $A$.
Notice that the factor $(1/mn)$ corresponds to the usual normalization by 
the number of entries and the factor $(1/\Mmax^2)$ corresponds to 
the maximum size of the matrix entries 
where $M$ satisfies $|M_{i,j}|\le\Mmax$ for all $i$ and $j$.
\begin{thm}\label{thm:Main}
Assume $M$ to be a rank $r$ 
 matrix of dimension $n\alpha\times n$ that satisfies 
$|M_{i,j}| \le \Mmax $ for all $i,j$. 
Then with probability larger than $1-1/n^3$
\begin{eqnarray}
\frac{1}{mn\Mmax^2}||M-\T_r(\tM^E)||_{\rm F}^2\le  C\,\frac{\alpha^{3/2}rn}{|E|}\, ,
\label{eq:Main}
\end{eqnarray}
for some numerical constant $C$.
\end{thm}
This theorem is proved in Section \ref{sec:result}.

Notice that the top $r$ singular values and singular vectors
of the sparse matrix $\tM^E$ can be computed efficiently 
by subspace iteration \cite{Berry92largescale}.
Each iteration requires $O(|E|r)$ operations. As proved in Section
\ref{sec:result}, the $(r+1)$-th singular value is smaller than 
one half of the  $r$-th one. As a consequence, subspace iteration converges 
exponentially. A simple calculation shows that $O(\log n)$
iterations are sufficient to ensure the error bound mentioned.

The `cleaning' step in the above pseudocode improves
systematically over $\T_r(\tM^E)$ and, for large enough $|E|$,
reconstructs $M$ exactly.
\begin{thm}\label{thm:Main2}
Assume $M$ to be a rank $r$ 
matrix that satisfies the incoherence conditions A1 and A2 with $(\mu_0,\mu_1)$.
Let $\mu = \max\{\mu_0 , \mu_1\}$.
Further, assume $\Sigma_{\rm min}\le \Sigma_1,\dots,\Sigma_{r}\le 
\Sigma_{\rm max}$ with $\Sigma_{\rm min}, \Sigma_{\rm max}$ bounded 
away from $0$ and $\infty$.
Then there exists a numerical constant 
$C'$ such that, if 
\begin{eqnarray}
|E|\ge  C' n r \sqrt{\alpha}\, \Big(\frac{\Smax}{\Smin}\Big)^2 
\max\Big\{\mu_0 \log n \, , \, \mu^2 r \sqrt{\alpha} \Big(\frac{\Sigma_{\rm max}}{\Sigma_{\rm min}}\Big)^4\Big\}\, ,
\end{eqnarray}
then the cleaning procedure in
{\sc Spectral Matrix Completion} converges, with high probability,
to the matrix $M$.
\end{thm}
This theorem is proved in Section \ref{sec:result2}.
The basic intuition is that, for $|E|\ge  C'(\alpha) nr\, \max\{\log n, r\}$,
$T_r(\tM^E)$ is so close to $M$ that the cost function is well approximated
by a quadratic function.

Theorem \ref{thm:Main} is optimal: the number of degrees of
freedom in $M$ is of order $nr$, without the same number
of observations is impossible to fix them. 
The extra $\log n$ factor in Theorem \ref{thm:Main2} is due to a
coupon-collector effect \cite{CaR08,KMO08,KOM09}:
it is necessary that $E$ contains at least one entry per row
and one per column and this happens only for $|E|\ge C n\log n$.
As a consequence, for rank $r$ bounded, Theorem~\ref{thm:Main2}
is optimal. It is suboptimal by a polylogarithmic factor for
$r = O(\log n)$.
%
%
\subsection{Related work}

Beyond collaborative filtering, low rank models are
used for clustering, information retrieval, 
machine learning, and image processing.
In \cite{Fazel}, the NP-hard problem of finding a matrix of minimum rank
satisfying a set of affine constraints was addresses through
convex relaxation. This problem is analogous to the problem of
finding the sparsest vector satisfying a set of affine constraints,
which is at the heart of \emph{compressed sensing} 
\cite{Donoho,CandesRombergTao}. 
The connection with compressed sensing was emphasized in
\cite{RFP}, that provided performance guarantees under appropriate conditions 
on the constraints. 

In the case of collaborative filtering, we are interested
in finding a matrix $M$ of minimum rank that matches the known 
entries $\{M_{ij}:\, (i,j)\in E\}$. Each known entry thus provides 
an affine constraint. Cand\`es and Recht \cite{CaR08} 
introduced the incoherent model for $M$. Within this model, they proved
that, if $E$ is random, the convex relaxation
correctly reconstructs  $M$ as long as $|E|\ge C\, r\,n^{6/5}\log n$. 
On the other hand, from a purely information theoretic 
point of view (i.e. disregarding algorithmic considerations),
it is clear that $|E| = O(n\,r)$ observations should allow to
reconstruct $M$ with arbitrary precision.
Indeed this point was raised in \cite{CaR08} and proved in \cite{KMO08},
through a counting argument. 

The present paper describes an efficient 
algorithm that reconstructs a rank-$r$ matrix from $O(n\, r)$
random observations. The most complex component of our algorithm 
is the SVD in step $2$.
We were able to treat realistic data sets with 
$n\approx 10^5$. This must be compared
with the $O(n^4)$ complexity of semidefinite programming
\cite{CaR08}.

Cai, Cand\`es and Shen \cite{CCS08} recently proposed a low-complexity 
procedure to solve the convex program posed in \cite{CaR08}.
Our spectral method is akin to a single step of this procedure,
with the important novelty of the trimming step that 
improves significantly its performances. Our analysis techniques might
provide a new tool for characterizing the convex relaxation as well.

Theorem \ref{thm:Main} can also be compared with a copious line of
work in the theoretical computer science literature
\cite{FKV,AFK01,AchlioptasRank}. An important motivation in 
this context is the development of fast algorithms for low-rank
approximation. In particular, Achlioptas and McSherry \cite{AchlioptasRank}
prove a theorem analogous to \ref{thm:Main}, but holding only
for $|E|\ge (8\log n)^4n$ (in the case of square matrices).

A short account of our results was submitted  to 
the 2009 International Symposium on Information Theory \cite{KOM09}. 
While the present paper was under completion, C\'andes and Tao 
posted online a preprint proving a theorem analogous to 
\ref{thm:Main2} \cite{CandesTaoMatrix}.
Once more, their approach is substantially different from ours.
%
%
\subsection{Open problems and future directions}

It is worth pointing out some limitations of our results,
and interesting research directions:

\vspace{0.1cm}

\emph{1. Optimal RMSE with $O(n)$ entries.}  
Numerical simulations with the {\sc Spectral Matrix Completion} 
algorithm suggest that the RMSE decays much  faster with 
the number of observations per degree of freedom $(|E|/nr)$, than indicated by
Eq.~(\ref{eq:Main}). This improved behavior is a consequence of the
cleaning step in the algorithm. 
It would be important to characterize the decay of 
RMSE with $(|E|/nr)$. 

\vspace{0.1cm}

\emph{2. Threshold for exact completion.} As pointed out,
Theorem \ref{thm:Main2} is order optimal for $r$ bounded. 
It would nevertheless be useful to derive quantitatively sharp 
estimates in this regime. A systematic numerical study was initiated in 
\cite{KMO08}. It appears that available theoretical estimates
(including the recent ones in \cite{CandesTaoMatrix}) are 
for larger values of the rank, we 
expect that our arguments can be strenghtened to prove exact reconstruction 
for $|E|\ge C'(\alpha)nr\log n$ for all values of $r$.

\vspace{0.1cm}

\emph{3. More general models.}
The model studied here and introduced in  \cite{CaR08}
presents obvious limitations. In applications to collaborative
filtering, the subset of observed entries $E$ is far
from uniformly random. 
A recent paper \cite{Rigidity} investigates the uniqueness of the 
solution of the matrix completion problem for general sets $E$. 
In applications to fast low-rank approximation,
it would be desirable to consider non-incoherent matrices
as well (as in  \cite{AchlioptasRank}).
  
%
%
\section{Incoherence property and some notations}\label{sec:coherence}

In order to formalize the notion of incoherence, 
we write $U=[u_1,u_2,\dots,u_r]$ and $V= [v_1,v_2,\dots,v_r]$
for the columns of the two factors, with $||u_i||=\sqrt{m}$, 
$||v_i|| = \sqrt{n}$ and $u_i^Tu_j=0$, $v_i^Tv_j=0$ for $i\neq j$
(there is no loss of generality in this, since normalizations
 can be adsorbed by redefining $\Sigma$).
We shall further write $\Sigma = \diag(\Sigma_1,\dots,\Sigma_r)$
with $\Sigma_1\ge \Sigma_2\ge\cdots \ge\Sigma_r > 0$. 
 
The matrices $U$, $V$ and $\Sigma$ will be said to be
$(\mu_0,\mu_1)$-\emph{incoherent} if they satisfy 
the following properties:
\begin{itemize}

\item[{\bf A1.}] For all  $i\in [m]$, $j\in [n]$, we have  
           $\sum_{k=1}^{r}{U_{i,k}^2} \le \mu_0 r$,
           $\sum_{k=1}^{r}{V_{i,k}^2} \le \mu_0 r $.
\item[{\bf A2.}] For all $i\in[m]$, $j\in [n]$,
we have $|\sum_{k=1}^{r}{U_{i,k}(\Sigma_k/\Sigma_1)V_{j,k}}|\leq\mu_1r^{1/2}$.
\end{itemize}
Apart from difference in normalization, these assumptions coincide with the ones in \cite{CaR08}.

Notice that the second incoherence assumption A2 
implies the bounded entry condition in Theorem 
\ref{thm:Main} with $\Mmax=\mu_1r^{1/2}$.
In the following, whenever we write that a property $A$ holds 
with high probability (w.h.p.), we mean that there exists a function
$f(n) = f(n;\alpha)$ such that 
$\prob(A)\ge 1-f(n)$ and $f(n)\to 0$.
In the case of exact completion (i.e. in the proof of Theorem 
\ref{thm:Main2}) $f(\,\cdot\, )$ can also depend on 
$\mu_0$, $\mu_1$, $\Sigma_{\rm min}$, $\Sigma_{\rm max}$, and $f(n)\to 0$ for 
$\mu_0,\mu_1,\Sigma_{\rm min},\Sigma_{\rm max}$ bounded away from $0$ and $\infty$.

Probability is taken with respect to the uniformly random 
subset $E\subseteq [m]\times [n]$.
Define $\eps\equiv|E|/\sqrt{mn}$. In the case when $m=n$, 
$\eps$ corresponds to the average number of revealed entries per row or column.
Then, it is convenient to work with a model in which 
each entry is revealed independently with probability
$\eps/\sqrt{mn}$. Since, with high probability
$|E|\in [\eps\sqrt{\alpha}\, n- A\sqrt{n\log n},\eps\sqrt{\alpha}\, 
n+ A\sqrt{n\log n}]$, any guarantee on the algorithm performances that holds
within one model, holds within the other model as well
if we allow for a vanishing shift in $\eps$.

Notice that we can assume $m\geq n$, 
since we can always apply our theorem to the transpose of the matrix $M$.
Throughout this paper, therefore, we will assume $\alpha \geq 1$.
Finally, we will use $C$, $C'$ etc.
to denote numerical constants.

Given a vector $x\in\reals^n$, $||x||$ will denote its Euclidean norm. 
For a matrix $X\in\reals^{n\times n'}$, $||X||_F$ is its Frobenius norm,
and $||X||_2$ its operator norm (i.e. $||X||_2= \sup_{u\neq 0}||Xu||/||u||$). 
The standard scalar product between vectors or matrices will sometimes
be indicated by $\<x,y\>$ or $\<X,Y\>$, respectively.
Finally, we use the standard combinatorics notation
$[N]= \{1,2,\dots,N\}$ to denote the set of first $N$ integers.
%
%
\section{Proof of Theorem \ref{thm:Main} and technical results}
\label{sec:result}

As explained in the previous section, the crucial idea
is to consider the singular value decomposition
of the trimmed matrix $\tM^E$ 
instead of the original matrix $M^E$, as in Eq.~(\ref{eq:WrongSVD}).
We shall then redefine $\{\sigma_i\}$, $\{x_i\}$, $\{y_i\}$, by letting
\begin{eqnarray}
\tM^E = \sum_{i=1}^{\min(m,n)}\sigma_ix_iy_i^T\, .
\end{eqnarray} 
Here $||x_i||=||y_i||=1$, $x_i^Tx_j= y_i^Ty_j=0$ for $i\neq j$ 
and  $\sigma_1\ge\sigma_2\ge\dots\ge 0$.
Our key technical result is that, apart from a trivial rescaling,
these singular values are close to the ones of the full matrix $M$.
\begin{lemma}\label{lem:singularvalues}
There exists a numerical constant $C>0$ such that, 
with probability larger than $1-1/n^3$
\begin{eqnarray}
 \left|\frac{\sigma_q}{\eps}-\Sigma_q\right| \le C\Mmax\sqrt{\frac{\alpha}{\eps}} \;,  \label{eq:singularvalues}
\end{eqnarray}
where it is understood that $\Sigma_q=0$ for $q>r$.
\end{lemma}
This result generalizes a celebrated 
bound on the second eigenvalue of random graphs \cite{FKS89,FeO05}
and is illustrated in Fig.~\ref{fig:singularvalues}:
the spectrum of $\tM^E$ clearly reveals the rank-$3$ structure of $M$. 

As shown in Section \ref{sec:SingValues},
Lemma \ref{lem:singularvalues} is a direct consequence of 
the following estimate.
\begin{lemma}\label{lem:spectralnorm}
There exists a numerical constant $C>0$ such that, 
with probability larger than $1-1/n^3$
\begin{eqnarray}
 \left|\left|\frac{\eps}{\sqrt{mn}}M-\tM^E\right|\right|_2 
 \le C\Mmax\sqrt{\alpha\eps} \, .\label{eq:spectralnormbound}
\end{eqnarray}
\end{lemma}
The proof of this lemma is given in Section~\ref{sec:spectralnormbound}.

We will now prove Theorem \ref{thm:Main}.
\begin{proof} 
(Theorem \ref{thm:Main})
By triangle inequality
\begin{align*}
\left|\left|M - \T_r(\tM^E)\right|\right|_2 &\leq \left|\left|
\frac{\sqrt{mn}}{\epsilon} \tM^E - \T_r(\tM^E)\right|\right|_2 + \left|\left|M - \frac{\sqrt{mn}}{\epsilon} \tM^E\right|\right|_2 \\
&\leq  \sqrt{mn} {\sigma_{r+1}}/{\epsilon} + C\Mmax \sqrt{\alpha mn} /\sqrt{\epsilon}\\
&\leq  2C\Mmax\sqrt{\frac{\alpha mn}{\epsilon}}\; ,
\end{align*}
where we used Lemma \ref{lem:spectralnorm} for the second inequality and Lemma \ref{lem:singularvalues} for the last inequality. 
Now, for any matrix $A$ of rank at most $2r$, $||A||_{F}\le \sqrt{2r}||A||_2$,
whence
\begin{align*}
\frac{1}{\sqrt{mn}}\big|\big| M - \T_r(\tM^E) \big|\big|_F &\le\frac{\sqrt{2r}}{\sqrt{mn}}\big|\big| M - \T_r(\tM^E) \big|\big|_2\\ 
&\le C'\Mmax\sqrt{\frac{\alpha r}{\epsilon}}\, .
\end{align*}
The result follows by using $|E| = \eps \sqrt{mn}$.
\vspace{-1.cm}

\end{proof}

%
%
\section{Proof of Lemma \ref{lem:spectralnorm}}\label{sec:spectralnormbound}

%
%
We want to show that $|x^T(\tM^E - \frac{\eps}{\sqrt{mn}}M)y| \le C \Mmax \sqrt{\alpha\eps}$ 
for each $x \in \reals^m$, $y \in \reals^n$ such that $||x||=||y||=1$.  
Our basic strategy (inspired by \cite{FKS89}) will be the following:\\
$(1)$ Reduce  to $x$, $y$ belonging to discrete
 sets $T_m$, $T_n$;\\
$(2)$ Bound the contribution of light couples 
by applying union bound to these discretized sets, 
with a large deviation estimate
on the random variable $Z$, defined as $Z\equiv\sum_{L}{x_i\tM^E_{i,j}y_j} - \frac{\eps}{\sqrt{mn}}x^TMy $; \\
$(3)$ Bound the contribution of heavy couples 
using bound on the discrepancy of corresponding graph. \\

The technical challenge is that a worst-case bound
on the tail probability of $Z$
is not good enough, and we must keep track of its 
dependence on $x$ and $y$.
The definition of {\emph light} and {\em heavy couples} 
is provided in the following section.

%
%
\subsection{Discretization}

We define
\begin{eqnarray*}
T_{n} &=& \left\{ x\in\Big\{\frac{\Delta}{\sqrt{n}}\Z\Big\}^n\;:\; ||x||\leq 1 \right\} \;,\\
\end{eqnarray*}
Notice that $T_n\subseteq S_n \equiv\{x\in\reals^n:\,
||x||\le 1\}$. Next remark is proved in \cite{FKS89,FeO05}, 
and relates the original problem to the discretized one.
\begin{remark} \label{remark:discrete} 
Let $R\in\reals^{m\times n}$ be a matrix.
If $|x^T R y|\leq B$ for all $x\in T_m$ and $y\in T_n$, 
then $|x'^T R y'|\leq (1-\Delta)^{-2}B$ for all $x'\in S_m$ 
and $y'\in S_n$.
\end{remark}
Hence it is enough to show that, with high probability, 
$|x^T(\tM^E - \frac{\eps}{\sqrt{mn}}M)y| \le C\Mmax \sqrt{\alpha\eps}$ for all $x\in T_m$ and $y\in T_n$.

A naive approach would be to apply concentration inequalities
directly to the random variable $x^T(\tM^E - \frac{\eps}{\sqrt{mn}}M)y$.
This fails because the vectors $x$, $y$ can contain entries
that are much larger than the typical size $O(n^{-1/2})$. 
We thus separate two contributions.
The first contribution is due to
\emph{light couples} $L\subseteq [m]\times [n]$, defined as
\begin{eqnarray*}
   L=\left\{(i,j)\;:\;|x_i M_{ij} y_j|\leq \Mmax\left(\frac{\eps}{mn}\right)^{1/2} \right\}\;.
\end{eqnarray*}
The second contribution is due to its complement $\bL$, 
which we call \emph{heavy couples}. 
We have
\begin{eqnarray}
\left| x^T \left( \tM^E - \frac{\eps}{\sqrt{mn}}M\right) y \right| \le 
\left| \sum_{(i,j)\in L}x_i \tM_{ij}^E y_j - \frac{\eps}{\sqrt{mn}}x^TMy\right|
+\left| \sum_{(i,j)\in \bL}x_i \tMij y_j \right|
\end{eqnarray}
In the next two subsections, we will prove that both
contributions are upper bounded by $C\Mmax\sqrt{\alpha\eps}$
for all $x\in T_m$, $y\in T_n$. Applying Remark \ref{remark:discrete} to 
$|x^T(\tM^E - \frac{\eps}{\sqrt{mn}}M)y|$, this proves the thesis.
%
%
\subsection{Bounding the contribution of light couples}\label{subsec:light}

Let us define the subset of row and column indices which
have not been trimmed as $\cA_l$ and $\cA_r$:
\begin{eqnarray*}
 \cA_l&=&\{i\in [m]\;:\;\deg(i) \leq \frac{2\eps}{\sqrt{\alpha}}\}\,,\\
 \cA_r&=&\{j\in [n]\;:\;\deg(j) \leq 2\eps\sqrt{\alpha}\}\,,
\end{eqnarray*}
where $\deg(\cdot)$ denotes the degree 
(number of revealed entries) of a row or a column. Notice that 
$\cA= (\cA_l,\cA_r)$
is a function of the random set $E$.
It is easy to get a rough estimate of the sizes of $\cA_l$, $\cA_r$.
\begin{remark}\label{remark:sizetrim}
 There exists $C_1$ and $C_2$ depending only on $\alpha$ such that, 
 with probability larger than $1-1/n^4$, $|\cA_l|\ge m-\max\{e^{-C_1\eps}m,C_2\alpha\}$,
and $|\cA_r|\ge n-\max\{e^{-C_1\eps}n,C_2\}$.
\end{remark}
For the proof of this claim, we refer to Appendix \ref{app:sizetrim}.
For any $E\subseteq [m]\times [n]$ and $A = (A_l,A_r)$
with $A_l\subseteq [m]$, $A_r\subseteq [n]$, we define $M^{E,A}$ 
by setting to zero the
entries of $M$ that are not in $E$, those whose row index is not in $A_l$,
and those whose column index not in $A_r$.
Consider the event
\begin{eqnarray}
\cH(E,A) = 
\left\{\exists\, x,y\,:\;\; 
\left| \sum_{(i,j)\in L}x_iM^{E,A}_{ij}y_j - \frac{\eps}{\sqrt{mn}}x^TMy 
\right| > C\Mmax\sqrt{\alpha\eps} \right\}\; ,\label{eq:EventDefinition1}
\end{eqnarray}
where it is understood that $x$ and $y$
belong, respectively, to $T_m$ and $T_n$.
Note that 
$\tM^E = M^{E,\cA}$, and hence we want to bound $\prob\{\cH(E,\cA)\}$. 
We proceed as follows
\begin{eqnarray}
\prob\left\{\cH(E,\cA)\right\} &=& \sum_{A} \prob\left\{\cH(E,A),\,\cA=A\right\} \nonumber\\
&\le & \sum_{\substack{|A_l|\ge m(1-\delta),\\|A_r|\ge n(1-\delta)}}
\prob\left\{\cH(E,A),\,\cA=A\right\} + \frac{1}{n^4} \nonumber\\
& \le & 2^{(n+m)H(\delta)}
\max_{\substack{|A_l|\ge m(1-\delta),\\|A_r|\ge n(1-\delta)}}
\prob\left\{\cH(E;A)\right\} + \frac{1}{n^4}\, ,\label{eq:ABound}
\end{eqnarray}
with $\delta\equiv\max\{e^{-C_1\eps},C_2\alpha\}$ 
and $H(x)$ the binary entropy function.

We are now left with the task of bounding
$\prob\left\{\cH(E;A)\right\}$ uniformly over $A$ where
$\cH$ is defined as in Eq.~(\ref{eq:EventDefinition1}).  
The key step consists in proving the following tail estimate

\begin{lemma}
Let $x\in S_m$, $y\in S_n$, 
$Z = \sum_{(i,j)\in L}x_i M^{E,A}_{ij}y_j - \frac{\eps}{\sqrt{mn}}x^TMy$, and assume 
$|A_l|\ge m(1-\delta)$, $|A_r|\ge n(1-\delta)$ with $\delta$ small enough. 
Then 
\begin{eqnarray*}
\P\left(Z > L\Mmax\sqrt{\eps}\right) \le 
\exp\Big\{-\frac{\sqrt{\alpha}(L-3)n}{2}\Big\}\;.
\end{eqnarray*}
\end{lemma}
\begin{proof}
We begin by bounding  the mean of $Z$ as follows
(for the proof of this statement we refer to Appendix \ref{app:lightexpectation}).
\begin{remark}\label{rem:lightexpectation}
$\left|\E\left[Z\right]\right| \le 2\Mmax \sqrt{\eps}$.
\end{remark}
For $A=(A_l,A_r)$, let $M^A$ be the matrix obtained from $M$ by setting to 
zero those entries whose row index is not in $A_l$,
and those whose column index not in $A_r$.
Define the potential contribution of the light couples $a_{ij}$ 
and independent random variables $Z_{ij}$ as 
\begin{eqnarray*}
   a_{ij} &=&  
   \left\{ \begin{array}{rl}
    x_i M^A_{ij} y_j & \text{if } |x_i M^A_{ij} y_j|\leq \Mmax\left({\eps}/{mn}\right)^{1/2},\\
    0             & \text{otherwise,}
    \end{array} \right. \\
Z_{ij} &=&  
   \left\{ \begin{array}{rl}
   a_{i,j} & \text{w.p. } {\eps}/{\sqrt{mn}},\\
    0         & \text{w.p. } 1-{\eps}/{\sqrt{mn}},
    \end{array} \right. 
\end{eqnarray*}
Let $Z_1 =\sum_{i,j}^{}{Z_{ij}}$
so that $ Z = Z_1 - \frac{\eps}{\sqrt{mn}}x^TMy$.
Note that $\sum_{i,j}{a_{ij}^2} \leq \sum_{i,j}{\left(x_i M^A_{ij} y_j\right)^2} \leq \Mmax^2$.
Fix $\lambda= \sqrt{mn}/2\Mmax\sqrt{\eps}$ so that $|\lambda a_{i,j}|\leq 1/2$, whence  
$e^{\lambda a_{ij}}-1 \leq \lambda a_{ij}+2(\lambda a_{ij})^2$.
It then follows that
\begin{eqnarray*}
\E[e^{\lambda Z}] &=   & 
\exp\Big\{\frac{\eps}{\sqrt{mn}}\Big(\sum_{i,j}\lambda a_{i,j}+2\sum_{i,j}(\lambda a_{i,j})^2 \Big) - \frac{\lambda\, \eps}{\sqrt{mn}}x^TMy \Big\}\\
 &\leq& \exp\Big\{
\lambda\E[Z] +\frac{\sqrt{mn}}{2}\Big\}\;.
\end{eqnarray*}
The thesis follows  by Chernoff bound $\prob(Z>a)\le e^{-\lambda a}\E[e^{\lambda Z}]$ after simple calculus. 
\end{proof}

Note that $\P\left(-Z > L\Mmax\sqrt{\eps}\right)$ can also be bounded analogously. 
We can now finish the upper bound on the light couples contribution.
Consider the error event Eq.~(\ref{eq:EventDefinition1}).
A simple volume calculation shows that $|T_m|\leq (10/\Delta)^m$. 
We can apply union bound over $T_m$ and $T_n$ to Eq.~(\ref{eq:ABound}) to obtain
\begin{eqnarray*}
\prob\{\cH(E,\cA)\}
 &\le& 2\cdot\,2^{(n+m)H(\delta)} \cdot \left(\frac{20}{\Delta}\right)^{n+m}
 e^{-\frac{(C-3)\sqrt{\alpha}n}{2}} + \frac{1}{n^4} \, \\
 &\le& \exp{\left\{\log 2 + (1+\alpha)\left(H(\delta)\log 2 +
 \log({20}/{\Delta})\right)n -\frac{(C-3)\sqrt{\alpha}n}{2} \right\}} + \frac{1}{n^4} \, .
\end{eqnarray*}
Hence, assuming $\alpha\ge 1$, there exists a numerical constant $C'$ such that, 
for $C > C'\sqrt\alpha$, the first term is of order $e^{-\Theta(n)}$, 
and this finishes the proof.

%
%
\subsection{Bounding the contribution of heavy couples}\label{subsec:heavy}

Let $Q$ be an $m\times n$ matrix with $Q_{ij}=1$ if 
$(i,j)\in E$ and $i\not\in\cA_r$, $j\not\in\cA_l$ 
(i.e. entry $(i,j)$ is not trimmed by our algorithm), and 
$Q_{ij} =0 $ otherwise.
Since $|M_{ij}|\le \Mmax$, the heavy couples satisfy 
$|x_iy_j|\ge \sqrt{\eps/mn}$. We then have
\begin{eqnarray*}
\left|\sum_{(i,j)\in \overline{L}}{x_i\tMij y_j}\right| 
&\leq& \Mmax \!\! \sum_{(i,j)\in \bL}Q_{ij}|x_iy_j|\\
&\leq& \Mmax \!\! \sum_{\substack{(i,j)\in E:\\ |x_iy_j|\ge \sqrt{\eps/mn}}} Q_{ij}|x_iy_j|\, .
\end{eqnarray*}
Notice that $Q$ is the adjacency matrix of a random bipartite graph
with vertex sets $[m]$ and $[n]$ and maximum degree bounded 
by $2\eps\max(\alpha^{1/2},\alpha^{-1/2})$. 
The following remark strengthens a result of \cite{FeO05}.
\begin{remark}\label{remark:HeavyPart}
Given vectors $x$, $y$, let $\bL' = \{(i,j): |x_iy_j|\ge C\sqrt{\eps/mn}\}$.
Then there exist a constant $C'$ such that,
$\sum_{(i,j)\in \bL'}Q_{ij}|x_iy_j|\le 
C'(\sqrt\alpha+\frac{1}{\sqrt\alpha})\sqrt{\eps}$, for all $x\in T_m$, $y\in T_n$
 with probability larger than $1-1/2n^3$.
\end{remark}
For the reader's convenience, a proof of this fact 
is proposed in Appendix \ref{app:HeavyPart}.
The analogous result in \cite{FeO05} (for the adjacency matrix of a 
non-bipartite graph) is proved to hold only with probability
larger than $1-e^{-C\eps}$. The stronger statement quoted here
can be proved using concentration of measure inequalities.
The last remark implies that for all $x\in T_m$, $y\in T_n$, and $\alpha\geq 1$, 
the contribution of heavy couples is
bounded by $C\Mmax\sqrt{\alpha\eps}$ for some numerical constant $C$ 
with probability larger than $1-1/2n^3$.

%
%
\section{Proof of Lemma \ref{lem:singularvalues}}\label{sec:SingValues}

Recall the variational principle for the singular values.
\begin{eqnarray}
\sigma_q &=& \min_{\substack{H,\dim(H)=n-q+1}}\;\max_{\substack{y\in H,||y||=1}}\;||\tM^Ey|| \label{eq:SVD_UB}\\
 &=& \max_{\substack{H,\dim(H)=q}}\;\min_{\substack{y\in H,||y||=1}}\;||\tM^Ey||\, .
\label{eq:SVD_LB}
\end{eqnarray}
Here $H$ is understood to be a linear subspace of $\reals^n$.

Using Eq.~(\ref{eq:SVD_UB}) with $H$ the orthogonal complement of
$\spn(v_1,\ldots,v_{q-1})$, we have, by Lemma \ref{lem:spectralnorm},
\begin{eqnarray*}
\sigma_q & \le & \max_{\substack{y\in H, ||y||=1} }   \big|\big|\tM^Ey\big|\big| \\
& \le & \frac{\epsilon}{\sqrt{mn}} \left( \max_{{y\in H, ||y||=1}}\big|\big|My\big|\big| \right) 
 + \max_{{y\in H, ||y||=||x||=1}} \left|x^T\left( \tM^E - \frac{\epsilon}{\sqrt{mn}}M\right)y\right|  \\
& \le & \epsilon \Sigma_q + C\Mmax\sqrt{\alpha\epsilon} \label{eq:UB}
\end{eqnarray*}

The lower bound is proved analogously, by using Eq.~(\ref{eq:SVD_LB}) with 
$H = \spn(v_1,\ldots,v_{q})$.

%
%
%
\section{Minimization on Grassmann manifolds and 
proof of Theorem \ref{thm:Main2}}\label{sec:result2}

The function $F(X,Y)$ defined in Eq.~(\ref{eq:MinimizeS}) and
to be minimized in the last part of the algorithm
can naturally be viewed as defined on Grassmann manifolds.
Here we  recall from \cite{Edelman}
a few important facts on the geometry of
Grassmann manifold and related optimization
algorithms. We then prove Theorem \ref{thm:Main2}. Technical calculations 
are deferred to Sections \ref{sec:LemmaQuad},
\ref{sec:LemmaGrad},  and to the appendices.

We recall that, for the proof of  Theorem \ref{thm:Main2}, it is assumed 
that $\Sigma_{\rm min}$, $\Sigma_{\rm max}$ are bounded away from $0$
and $\infty$. Numerical constants are denoted by $C, C'$ etc. Finally,
throughout this section, we use the notation $X^{(i)}\in\reals^r$ to refer to 
the $i$-th row of the matrix $X\in\reals^{m\times r}$ or 
$X\in\reals^{n\times r}$.
%
%
\subsection{Geometry of the Grassmann manifold}

Denote by $\Orth(d)$ the orthogonal group of
$d\times d$ matrices. The Grassmann manifold is defined as the quotient 
$\Grass(n,r) \simeq \Orth(n)/\Orth(r)\times \Orth(n-r)$. 
In other words, a point in
the manifold is the equivalence class 
of an $n\times r$ orthogonal matrix $A$
\begin{eqnarray}
[A] = \{AQ :\, Q\in \Orth(r)\}\, .
\end{eqnarray}
For consistency with the rest of the paper, we will assume the normalization
$A^TA = n\,\id$. To represent a point in $\Grass(n,r)$, we will use 
an explicit representative of this form. More abstractly, 
$\Grass(n,r)$ is the manifold of $r$-dimensional subspaces of $\reals^n$.

It is easy to see that $F(X,Y)$ depends on the matrices
$X$, $Y$ only through their equivalence classes $[X]$,  $[Y]$.
We will therefore interpret it as a function defined on
the manifold $\Manif(m,n) \equiv \Grass(m,r)\times 
\Grass (n,r)$:
\begin{eqnarray}
F: \Manif(m,n) &\to & \reals\, ,\\
([X],[Y])&\mapsto & F(X,Y)\, .
\end{eqnarray}
In the following, a point in this manifold will be represented 
as a pair $\Xm = (X,Y)$, with $X$ an $n\times r$ orthogonal matrix 
and $Y$ an $m\times r$ orthogonal matrix. Boldface symbols will be
reserved for elements of $\Manif(m,n)$ or of its tangent space,
and we shall use $\Um=(U,V)$ for the point corresponding to the
matrix $M=U\Sigma V^T$ to be reconstructed. 

Given  $\Xm = (X,Y)\in \Manif(m,n)$, the tangent space 
at $\Xm$ is denoted by $\Tang_{\Xm}$ and can be identified with 
the vector space of matrix pairs $\Wm = (W,Z)$, $W\in \reals^{m\times r}$,
$Z\in\reals^{n\times r}$ such that $W^TX=Z^TY=0$. The 
`canonical' Riemann metric on the Grassmann manifold
corresponds to the usual scalar product $\<W,W'\> \equiv \Trace(W^TW')$.
The induced scalar product on $\Tang_{\Xm}$
between $\Wm = (W,Z)$ and $\Wm' = (W',Z')$ is 
$\<\Wm,\Wm'\>=\<W,W'\>+\<Z,Z'\>$.

This metric induces a canonical notion of distance on $\Manif(m,n)$
which we denote by $d(\Xm_1,\Xm_2)$ (geodesic or arc-length distance). 
If $\Xm_1=(X_1,Y_1)$ and $\Xm_2=(X_2,Y_2)$ then  
\begin{eqnarray}
d(\Xm_1,\Xm_2)\equiv \sqrt{d(X_1,X_2)^2+d(Y_1,Y_2)^2}
\end{eqnarray}
where the arc-length distances $d(X_1,X_2)$,
$d(Y_1,Y_2)$ on the Grassmann manifold can be defined 
explicitly as follows. Let $\cos\theta = (\cos\theta_1,\dots,\cos\theta_r)$,
$\theta_i\in[-\pi/2,\pi/2]$ be the singular values of $X_1^TX_2/m$. Then
\begin{eqnarray}
d(X_1,X_2) = ||\theta||_2\, .
\end{eqnarray}
The $\theta_i$'s are called the `principal angles' between the 
subspaces spanned by the columns of $X_1$ and $X_2$.
It is useful to introduce two equivalent notions of distance:
\begin{align}
d_{\rm c}(X_1,X_2) & = \frac{1}{\sqrt{n}}
\min_{Q_1,Q_2\in\Orth(r)}||X_1Q_1-X_2Q_2||_F\, &&
\mbox{(chordal distance),}\\
d_{\rm p}(X_1,X_2) & =   \frac{1}{\sqrt{2}n}
||X_1X_1^T-X_2X_2^T||_F\, &&\mbox{(projection distance).}
\end{align}
Notice that $d_{\rm c}$ and $d_{\rm p}$ do not depend on the specific
representatives $X_1$, $X_2$, but only on the equivalence classes
$[X_1]$ and $[X_2]$. Distances on $\Manif(m,n)$ are defined 
through Pythagorean theorem, e.g. $d_{\rm c}(\Xm_1,\Xm_2)
= \sqrt{d_{\rm c}(X_1,X_2)^2+d_{\rm c}(Y_1,Y_2)^2}$.
\begin{remark}\label{remark:DistancesGrassmann}
The geodesic, chordal and projection distance are equivalent, namely
\begin{eqnarray}
\frac{1}{\pi}d(X_1,X_2)\le \frac{1}{\sqrt{2}}\, d_{\rm c}(X_1,X_2)\le d_{\rm p}(X_1,X_2)
\le d_{\rm c}(X_1,X_2)\le d(X_1,X_2) \, .
\end{eqnarray}
\end{remark}
For the reader's convenience, a proof of this fact is 
proposed in Appendix  \ref{app:Distance}.

An important remark is that geodesics with respect to the canonical
Riemann metric admit an explicit and efficiently computable form.
Given $\Um\in\Manif(m,n)$, $\Wm\in\Tang_{\Um}$
the corresponding geodesic is a curve $t\mapsto \Xm(t)$,
with $\Xm(t) = \Um+\Wm t+O(t^2)$ which minimizes arc-length. 
If $\Um = (U,V)$ and $\Wm = (W,Z)$ then $\Xm(t) = (X(t),Y(t))$ where
$X(t)$ can be expressed in terms of the singular value 
decomposition $W=L\Theta R^T$ \cite{Edelman}:
\begin{eqnarray} 
X(t) = UR\cos(\Theta t) R^T + L\sin(\Theta t) R^T\, ,\label{eq:Geodesic}
\end{eqnarray} 
which can be evaluated in time of order $O(n r)$. An analogous 
expression holds for $Y(t)$. 
%
%
\subsection{Gradient and incoherence}

The gradient of $F$ at 
$\Xm$ is the vector $\grad F(\Xm)\in \Tang_{\Xm}$ such that,
for any smooth curve $t\mapsto\Xm(t)\in\Manif(m,n)$
with $\Xm(t) = \Xm+\Wm\, t+O(t^2)$,  one has
\begin{eqnarray}
F(\Xm(t)) = F(\Xm) + \<\grad F(\Xm),\Wm\>\, t+O(t^2)\, .
\end{eqnarray}
In order to write an explicit representation of the gradient of 
our cost function $F$, it is 
convenient to introduce the projector operator
\begin{eqnarray}
\cP_E(M)_{ij} = \left\{\begin{array}{ll}
M_{ij} & \mbox{ if $(i,j)\in E$,}\\
0 & \mbox{otherwise.}
\end{array}\right.\label{eq:ProjectorDef}
\end{eqnarray}
The two components of the gradient are then
\begin{eqnarray}
\grad F(\Xm)_{X} &=& \cP_E(XSY^T-M)YS^T-XQ_X\, ,\label{eq:gradX}\\
\grad F(\Xm)_{Y} &=& \cP_E(XSY^T-M)^TXS-YQ_Y\, ,\label{eq:gradY}
\end{eqnarray}
where $Q_X,Q_Y\in\reals^{r\times r}$ are determined by the 
condition $\grad F(\Xm)\in \Tang_{\Xm}$. This yields
\begin{eqnarray}
Q_X & = & \frac{1}{m} X^T\cP_E(M-XSY^T)YS^T\, ,\\
Q_Y & = &  \frac{1}{n} Y^T\cP_E(M-XSY^T)^TXS\, .
\end{eqnarray}
%
%
\subsection{Algorithm}

At this point the gradient descent algorithm is fully specified.
It takes as input the factors of $\T_r(\tM^E)$, to be denoted as 
$\Xm_0 = (X_0, Y_0)$, and minimizes a regularized cost function
\begin{eqnarray}
 \tF(X,Y) &=& F(X,Y) + \rho \, G(X,Y)\\
 &\equiv& F(X,Y) + \rho\sum_{i=1}^mG_1\left(\frac{||X^{(i)}||^2}{3\mu_0r} \right) 
  + \rho\sum_{j=1}^nG_1\left(\frac{||Y^{(j)}||^2}{3\mu_0r}\right)\, ,
\label{eq:Gdef}
\end{eqnarray}
where $X^{(i)}$ denotes the $i$-th row of $X$, and
$Y^{(j)}$ the $j$-th row of $Y$.
The role of the regularization is to force $\Xm$ to remain incoherent during 
the execution of the algorithm.
\begin{eqnarray}
G_1(z) = \left\{\begin{array}{ll}
0 & \mbox{ if $z\le 1$,}\\
e^{(z-1)^2}-1 & \mbox{ if $z\ge 1$.} 
\end{array}\right.
\end{eqnarray}
We will take $\rho= n\eps$. Notice that $G(X,Y)$ is
again naturally defined on the Grassmann manifold, i.e.
$G(X,Y) = G(XQ,YQ')$ for any $Q,Q'\in\Orth(r)$.

Let 
\begin{eqnarray}
\Co(\mu') \equiv \left\{(X,Y)
\mbox{ such that }||X^{(i)}||^2\le \mu' r,\;
||Y^{(j)}||^2\le \mu' r \mbox{ for all } \,i\in [m],\,j\in [n]\right\}\, .
\end{eqnarray}
We have $G(X,Y) = 0$ on $\Co(3\mu_0)$.
Notice that $\Um\in\Co(\mu_0)$ by the incoherence property.
Also, by the following remark proved in Appendix \ref{app:Distance}, we can assume that $\Xm_0\in\Co(3\mu_0)$.

\begin{remark}\label{rem:rescaling}
Let $U, X \in \reals^{n \times r}$ with $U^TU = X^TX = n\id$ and $U \in \Co(\mu_0)$ and $d(X,U) \le \delta \le \frac{1}{16}$. Then there exists $X'' \in \reals^{n \times r}$ such that $X''^TX'' = n\id$, $X'' \in \Co(3\mu_0)$ and $d(X'',U) \le 4 \delta$. Further, such an $X''$ can be computed in a time of $O(nr^2)$.
\end{remark}


\vspace{0.3cm}

\begin{tabular}{ll}
\hline
\multicolumn{2}{l}{ {\sc Gradient descent}( matrix $M^E$, factors 
$\Xm_0$ )}\\
\hline
1: & For $k=0,1,\dots$ do: \\
2: &\hspace{0.2cm} Compute $\Wm_k = \grad \tF(\Xm_k)$;\\
4: &\hspace{0.2cm} Let $t\mapsto \Xm_k(t)$ be the geodesic 
with $\Xm_k(t) = \Xm_k+\Wm_k t+O(t^2)$;\\
5: &\hspace{0.2cm} Minimize $t\mapsto \tF(\Xm_k(t))$ for $t\ge 0$, 
subject  to $d(\Xm_k(t),\Xm_0)\le \gamma$;\\
6: &\hspace{0.2cm} Set $\Xm_{k+1} = \Xm_k(t_k)$ where $t_k$ is the minimum location;\\
7: & End For.\\
\hline
\end{tabular}

\vspace{0.3cm}

In the above, $\gamma$ must be set in such a way that 
$d(\Um,\Xm_0)\le \gamma$. The next remark determines the correct scale.
\begin{remark}\label{remark:Near}
Let $U,X\in\reals^{m\times r}$ with $U^TU=X^TX=m\id$,
$V,Y\in\reals^{n\times r}$ with $V^TV=Y^TY=n\id$,
and $M= U\Sigma V^T$, $\hM=XSY^T$ for $\Sigma 
= \diag(\Sigma_1,\dots,\Sigma_r)$ and $S\in\reals^{r\times r}$.
If $\Sigma_1,\dots,\Sigma_r\ge \Sigma_{\rm min}$, then
\begin{eqnarray}
d_{\rm p}(U,X)\le \frac{1}{\sqrt{2\alpha}n\Sigma_{\rm min}}\, ||M-\hM||_F\, \;\;\; ,\;\;\;\;\;
d_{\rm p}(V,Y)\le  \frac{1}{\sqrt{2\alpha}n\Sigma_{\rm min}}\, ||M-\hM||_F
\end{eqnarray}
\end{remark}
As a consequence of this remark and Theorem \ref{thm:Main},
we can assume that $d(\Um,\Xm_0)\le C (\frac{\Smax}{\Smin})\,\frac{\mu_1 r\sqrt{\alpha}}{\sqrt{\eps}}$.
We shall then set $\gamma = C' (\frac{\Smax}{\Smin})\,\frac{\mu_1 r\sqrt{\alpha}}{\sqrt{\eps}}$ (the value of $C'$ is set in the
course of the proof).

Before passing to the proof of Theorem  \ref{thm:Main2}, it is worth 
discussing a few important points concerning the gradient descent algorithm.
\begin{enumerate}
\item[$(i)$] The appropriate choice of $\gamma$ might seem to pose a difficulty.
In reality, this parameter is introduced only to simplify the proof.
We will see that the constraint $d(\Xm_k(t),\Xm_0)\le \gamma$ is, with high 
probability, never saturated.
\item[$(ii)$] Indeed, the line minimization instruction 5
(which might appear complex to implement)
can be replaced by a standard step selection procedure, such as the
one in \cite{Armijo}.
\item[$(iii)$] Similarly, there is no need to know the actual value of $\mu_0$
in the regularization term. One can start with $\mu_0=1$ and then
repeat the optimization doubling it at each step.
\item[$(iv)$] The Hessian of $F$ can be computed explicitly as well.
This opens the way to quadratically convergent minimization
algorithms (e.g. the Newton method).
\end{enumerate}
%
%
\subsection{Proof of Theorem \ref{thm:Main2}}

The proof of Theorem \ref{thm:Main2} breaks down in two lemmas.
The first one implies that, in a sufficiently small neighborhood
of $\Um$, the function $\Xm\mapsto F(\Xm)$ is 
well approximated by a parabola.
\begin{lemma}\label{lemma:Quadratic}
There exists numerical constants $C_0, C_1, C_2$
such that the following happens.
Assume $\eps\ge C_0\mu_0\sqrt{\alpha}\,r \max\{\log n ; \mu_0 r \sqrt{\alpha} (\Smax/\Smin)^4\}$  and $\delta\le \Sigma_{min}/C_0\Sigma_{\rm max}$.
Then 
\begin{eqnarray}
 C_1 \sqrt{\alpha} \Sigma_{\rm min}^2\, d(\Xm,\Um)^2 + C_1 \sqrt{\alpha}\, ||S-\Sigma||_{F}^2
\le \frac{1}{n\eps}\, F(\Xm)\le  C_2\sqrt{\alpha}\Sigma_{\rm max}^2
d(\Xm,\Um)^2\label{eq:Quadratic}
\end{eqnarray}
for all $\Xm\in\Manif(m,n)\cap \Co(4\mu_0)$ such that $d(\Xm,\Um)\le \delta$,
with probability at least $1-1/n^4$.
Here $S\in\reals^{r\times r}$ is the matrix realizing the
minimum in Eq.~(\ref{eq:MinimizeS}).
\end{lemma}

The second Lemma implies that $\Xm\mapsto F(\Xm)$ does not have any
other stationary point (apart from $\Um$) within such a neighborhood.
\begin{lemma}\label{lemma:Gradient}
There exists numerical constants $C_0, C$
such that the following happens.
Assume $\eps\ge C_0  \mu_0 r\sqrt{\alpha} (\Sigma_{\rm max}/\Sigma_{\rm min})^2 \max \{\log n ; \mu_0 r \sqrt{\alpha} (\Smax/\Smin)^4\} $ and 
$\delta\le \Sigma_{min}/C_0\Sigma_{\rm max}$.
Then 
\begin{eqnarray*}
||\grad \tF(\Xm)||^2 \ge C\, n\eps^2 \Sigma_{\rm min}^4\, d(\Xm,\Um)^2
\end{eqnarray*}
for all $\Xm\in\Manif(m,n)\cap \Co(4\mu_0)$ such that $d(\Xm,\Um)\le \delta$,
with probability at least $1-1/n^4$.
\end{lemma}

We can now prove Theorem \ref{thm:Main2}.
\begin{proof}(Theorem \ref{thm:Main2})
Let $\delta>0$ be such that Lemma \ref{lemma:Quadratic} and
Lemma \ref{lemma:Gradient} are verified, and
$C_1$, $C_2$ be defined as in Lemma \ref{lemma:Quadratic}.
We further assume $\delta \le \sqrt{(e^{1/9}-1)/C_2}$.
Take $\eps$ large enough such that, 
$ d(\Um,\Xm_0)\le \min(1,(C_1/C_2)^{1/2}(\Sigma_{\rm min}/\Sigma_{\rm max}))\delta/10 $.
Further, set the algorithm parameter to 
$\gamma = \delta/4$.

We make the following claims:
\begin{enumerate}
\item $\Xm_k\in\Co(4\mu_0)$ for all $k$. 

Indeed $\Xm_0\in\Co(3\mu_0)$
whence $\tF(\Xm_0)=F(\Xm_0) \le C_2\sqrt{\alpha} n\eps\Sigma_{\rm max}^2\,\delta^2$.
The claim follows because $\tF(\Xm_k)$ is non-increasing 
and $\tF(\Xm)\ge \rho\, G(X,Y)\ge n\eps \sqrt{\alpha} \Smax^2 (e^{1/9}-1)$ for $\Xm\not\in  
\Co(4\mu_0)$, where we choose $\rho$ to be $n\eps \sqrt{\alpha} \Smax^2$.
\item $d(\Xm_k,\Um)\le \delta/10$ for all $k$.

Since we set $\gamma=\delta/4$, by triangular inequality, 
we can assume to have $d(\Xm_k,\Um)\le \delta/2$.
Since  $d(\Xm_0,\Um)^2\le (C_1\Sigma_{\rm min}^2/C_2\Sigma_{\rm max}^2)(\delta/10)^2$, we have 
$\tF(\Xm)\ge F(\Xm)\ge F(\Xm_0)$ for  all $\Xm$ such that
$d(\Xm,\Um)\in[\delta/10,\delta]$.
Since  $\tF(\Xm_k)$ is non-increasing and $\tF(\Xm_0) = F(\Xm_0)$,
the claim follows.
\end{enumerate}

Notice that, by the last observation, the constraint 
$d(\Xm_k(t),\Xm_0)\le \gamma$ is never saturated, and therefore our 
procedure is just gradient descent with exact line search.
Therefore by \cite{Armijo} this must converge to the unique
stationary point of $\tF$
in $\Co(4\mu_0)\cap\{\Xm:\, d(\Xm,\Um)\le \delta/10\}$, 
which, by Lemma \ref{lemma:Gradient},
is $\Um$.  
\end{proof}
%
%
\section{Proof of Lemma \ref{lemma:Quadratic}}
\label{sec:LemmaQuad}

\subsection{A random graph Lemma}

The following Lemma will be used several times in the following.
\begin{lemma}\label{lemma:xAy}
There exist two numerical constants $C_1, C_2$ suct that 
the following happens.
If $\eps \ge C_1\log n$ then, with probability larger than $1-1/n^5$, 
\begin{eqnarray}
\sum_{(i,j)\in E}x_iy_j \le \frac{C_2\eps}{n\sqrt{\alpha}}\; ||x||_1||y||_1
+ C_2\sqrt{\alpha \eps} ||x||_2\, ||y||_2\, .
\end{eqnarray}
for all $x\in\reals^m$, $y\in\reals^n$.
\end{lemma}
\begin{proof}
Write $x_i = x_0+x'_i$ where
$\sum_ix_i'=0$. Then 
\begin{eqnarray}
\sum_{(i,j)\in E}x_iy_j = 
x_0\sum_{j\in [n]}\deg(j) y_j
+ \sum_{(i,j)\in E}x'_iy_j\, ,
\end{eqnarray}
where we recall that $\deg(j) = \{i\in[m]:\, $ such that $(i,j)\in E\}$.
Further $|x_0| =|\sum_ix_i/m|\le ||x||_1/m$. 
The first term is upper bounded by 
\begin{eqnarray}
x_0\max_{j\in n}\deg(j) ||y||_1
\le \max_{j\in n}\deg(j) ||x||_1||y||_1/m\, . 
\end{eqnarray}
For $\eps \geq C_1\log n$, with probability larger than $1-1/2n^5$,
the maximum degree is bounded by $(9/C_1)\sqrt\alpha\eps$ 
which is of same order as the average degree.
Therefore this term is at most
$C_2\sqrt\alpha\eps ||x||_1||y||_1/m$. 

The second term is upper bounded by $C_2\sqrt{\alpha \eps} ||x'||_2||y||_2$
using  Theorem 1.1
in \cite{FeO05} or, equivalently,  Theorem \ref{lem:singularvalues}
in the case $r=1$ and $\Mmax=1$. It can be shown to hold with probability larger than $1-1/2n^5$
with a large enough numerical constant $C_2$.
The thesis follows because $||x'||_2\le||x||_2$.
\end{proof}
%
%
\subsection{Preliminary facts and estimates}
\label{sec:PreliminariesProof}

This subsection contains some remarks 
that will be useful in the proof of Lemma \ref{lemma:Gradient}
as well.

Let $\Wm = (W,Z)\in \Tang_{\Um}$, and $t\mapsto (X(t),Y(t))$ be the geodesic 
such that $(X(t),Y(t)) = (U,V)+(W,Z)t+O(t^2)$.
By setting $(X,Y)=(X(1),Y(1))$, we establish a one-to-one
correspondence between the points $\Xm$ as in the statement and
a neighborhood of the origin in $\Tang_{\Um}$. If
we let $W = L\Theta R^T$ be the singular value decomposition of 
$W$ (with $L^TL=m\id$ and $R^TR=\id$), 
the explicit expression for geodesics in Eq.~(\ref{eq:Geodesic})
yields
\begin{eqnarray}
X= U+\oW\, , \;\;\;\;\;\; \oW = UR(\cos\Theta-\id)R^T + L\sin\Theta R^T\, .
\label{eq:FiniteDiff}
\end{eqnarray}
An analogous expression can obviously be written for $Y = V+\oZ$.
Notice that, by the equivalence between chordal and canonical
distance, Remark \ref{remark:DistancesGrassmann}, we have
\begin{eqnarray}
\frac{1}{m}||\oW||_F^2+\frac{1}{n}||\oZ||_F^2 \le 2\, d(\Um,\Xm)^2\, .
\label{eq:DistanceBoundAA}
\end{eqnarray}
\begin{remark}\label{rem:IncoherenceBound}
If $\Um\in \Co(\mu_0)$ and $\Xm\in \Co(4\mu_0)$, then $(\oW,\oZ)\in\Co(10\mu_0)$
and $\Wm = (W,Z)\in\Co(5\pi^2\mu_0/2)$.
\end{remark}
\begin{proof}
The first fact follows from
$||\oW^{(i)}||^2\le 2||X^{(i)}||^2+2||U^{(i)}||^2$.
In order to prove $\Wm\in\Co(5\pi^2\mu_0/2)$, we notice that
\begin{eqnarray*}
||W^{(i)}||^2&=& ||\Theta L^{(i)}||^2\le \frac{\pi^2}{4}||\sin\Theta L^{(i)}||^2\\
&\le&\frac{\pi^2}{4}||X^{(i)}-R\cos\Theta R^TU^{(i)}||^2\le 
\frac{\pi^2}{2}\Big(||X^{(i)}||^2+||U^{(i)}||^2\Big)\, .
\end{eqnarray*} 
The claim follows by showing a similar bound for $||Z^{(i)}||^2$.
\end{proof}

We next prove a simple a priori estimate.
\begin{remark}\label{rem:AprioriBound}
There exist numerical constants $C_1, C_2$ such that the following holds
with probability larger than $1-1/n^5$. 
If $\eps\ge C_1\log n$, then 
for any $(X,Y)\in \Co(\mu)$ and $S\in\reals^{r\times r}$,
\begin{eqnarray}
%
\sum_{(i,j)\in E}(XSY^T)_{ij}^2 \le 
C_2||S||_2^2 {\sqrt{\alpha}\,n\eps}\left(\frac{1}{m}||X||_F^2+\frac{1}{n}||Y||_F^2\right) 
\left(\frac{1}{m}||X||_F^2+\frac{1}{n}||Y||_F^2+\frac{\mu r \sqrt{\alpha} }{\sqrt{\eps}}\right)\, .
\end{eqnarray}
\end{remark}
\begin{proof}
Using Lemma \ref{lemma:xAy}, $\sum_{(i,j)\in E}(XSY^T)_{ij}^2$ 
is upper bounded by
\begin{align*}
&\sigma_{\rm max}(S)^2\sum_{a,b}\sum_{(i,j)\in E} X_{ia}^2Y_{jb}^2 \\
& \le \frac{C_2\eps}{n \sqrt{\alpha}} \sigma_{\rm max}(S)^2\sum_{i,j}
 ||X^{(i)}||^2||Y^{(j)}||^2 + C_2\sigma_{\rm max}(S)^2\sqrt{\alpha\eps}
 \Big(\sum_i ||X^{(i)}||^4\Big)^{1/2}\Big(\sum_j ||Y^{(j)}||^4\Big)^{1/2}\\
&\le  \frac{C_2\eps}{n \sqrt{\alpha}} \sigma_{\rm max}(S)^2\sum_{i,j}
 ||X^{(i)}||^2||Y^{(j)}||^2 + C_2\sigma_{\rm max}(S)^2\sqrt{\alpha\eps} \mu r
 \Big(\sum_{i} ||X^{(i)}||^2\Big)^{1/2}
 \Big(\sum_{j} ||Y^{(j)}||^2 \Big)^{1/2}\, \\
&\le C_2||S||_2^2 {\sqrt{\alpha}\,n\eps}\Big(\frac{1}{m}||X||_F^2+\frac{1}{n}||Y||_F^2\Big)^2
 + C_2||S||_2^2 \, \alpha\mu r \,n\sqrt{\eps} \Big(\frac{1}{m}||X||_F^2+\frac{1}{n}||Y||_F^2\Big)\, ,
\end{align*}
where in the second step we used the incoherence condition.
The last step follows from the inequalities $2ab \leq \alpha(a/\alpha+b)^2$ 
and $2ab\le \sqrt\alpha(a^2/\alpha+b^2)$. 
\end{proof}
%
%
\subsection{The proof}
\begin{proof}(Lemma \ref{lemma:Quadratic})
Denote by $S\in\reals^{r\times r}$ the matrix realizing the
minimum in Eq.~(\ref{eq:MinimizeS}). We will start by proving a lower bound
on $F(\Xm)$ of the form 
\begin{eqnarray}
\frac{1}{n\eps} F(\Xm) \ge 
C_1\sqrt{\alpha}\,\Sigma_{\rm min}^2\, d(\Xm,\Um)^2 + C_1\sqrt{\alpha}\,||S-\Sigma||_{F}^2 -
C_1'\sqrt{\alpha} \Sigma_{\rm max}d(\Xm,\Um)^2||S-\Sigma||_{F}\, ,
\label{eq:SloppyBound}
\end{eqnarray}
and an upper bound as in Eq.~(\ref{eq:Quadratic}).
Together, for $d(\Xm,\Um)\le \delta\le 1$, these imply  
$||S-\Sigma||_{F}^2\le C\Sigma_{\rm max}^2d(\Xm,\Um)^2$, whence the 
lower bound in Eq.~(\ref{eq:Quadratic}) follows for $\delta \leq \Sigma_{\rm min}/C_0\Sigma_{\rm max}$.

In order to prove the bound (\ref{eq:SloppyBound})
we write $X=U+\oW$, $Y=V+\oZ$, and
\begin{eqnarray*}
F(X,Y) & = &\frac{1}{2}\sum_{(i,j)\in E}(U(S-\Sigma)V^T
+US\oZ^T+\oW S V^T+\oW S\oZ^T)_{ij}^2\\
&\ge & \frac{1}{4}A^2 - \frac{1}{2}B^2\, 
\end{eqnarray*}
where we used the inequality 
$(1/2)(a+b)^2\ge (a^2/4)-(b^2/2)$, and defined
\begin{eqnarray*}
A^2 & \equiv &\sum_{(i,j)\in E}(U(S-\Sigma)V^T
+US\oZ^T+\oW S V^T)_{ij}^2\, ,\\
B^2 &\equiv &  \sum_{(i,j)\in E}(\oW S\oZ^T)_{ij}^2\, .
\end{eqnarray*}
Using Remark \ref{rem:AprioriBound}, and
Eq.~(\ref{eq:DistanceBoundAA}) we get 
\begin{align*}
B^2
&\le C \sqrt{\alpha} n\eps\, ||S||_2^2\left( d(\Xm,\Um)^2+
\frac{\mu_0 r \sqrt{\alpha} }{\sqrt{\eps}}\right)\, d(\Xm,\Um)^2\\
&\le 2C \sqrt{\alpha} n\eps \, \big(\Sigma_{\rm max}^2+||S-\Sigma||_F^2\big)
\left( \delta^2+\frac{\mu_0 r\sqrt{\alpha}}{\sqrt{\eps}}\right)\, d(\Xm,\Um)^2\, ,
\end{align*}
where the second inequality follows from the inequality 
$\sigma_{\rm max}(S)^2\le 2\Sigma_{\rm max}^2+
2\, ||S-\Sigma||_F^2$

By Theorem 4.1 in \cite{CaR08}, we have $A^2\ge (1/2)\E\{ A^2\}$
with probability larger than $1-1/n^5$ for $\eps \geq C\mu_0\sqrt{\alpha}\,r\log n$. Further 
\begin{eqnarray*}
\E\{ A^2\} & =& \frac{\eps}{\sqrt{mn}}
||U(S-\Sigma)V^T + US\oZ^T+\oW S V^T||_F^2\\
& = &  \frac{\eps}{\sqrt{mn}}||U(S-\Sigma)V^T||_F^2+\frac{\eps}{\sqrt{mn}} 
||US\oZ^T||^2_F +\frac{\eps}{\sqrt{mn}} ||\oW S V^T||^2_F\label{eq:EA2}\\
&&+\frac{2\eps}{\sqrt{mn}}
\<US\oZ^T,\oW S V^T\>+ \frac{2\eps}{\sqrt{mn}}\<U (S-\Sigma)V^T,\oW S V^T\> 
+ \frac{2\eps}{\sqrt{mn}}\<US\oZ^T,U(S-\Sigma) V^T\>\, .
\end{eqnarray*}
Let us call the absolute value of the six terms on the right hand side
$E_1$, \dots $E_6$.
A simple calculation yields
\begin{eqnarray}
E_1 &= &n\eps\sqrt{\alpha} ||S-\Sigma||_F^2\, ,\\
E_2+E_3 &\ge & n\eps\sqrt{\alpha} 
\sigma_{\rm min}(S)^2 \Big(\frac{1}{m}||\oW||^2_F+ \frac{1}{n}||\oZ||^2_F
\Big)
\ge  C'\sigma_{\rm min}(S)^2 n\eps \sqrt{\alpha} d(\Xm,\Um)^2\, .\label{eq:E2E3}
\end{eqnarray}
The absolute value of the 
fourth term can be written as
\begin{eqnarray*}
E_4 & = & \frac{2\eps}{n\sqrt{\alpha}} |\<US\oZ^T,\oW SV^T\>|
\le \frac{2\eps}{n\sqrt{\alpha}}\, \sigma_{\rm max}(S)^2 ||\oW^TU||_F||V^T\oZ||_F\\
&\le & \frac{2\eps\alpha}{n\sqrt{\alpha}}\sigma_{\rm max}(S)^2 (\frac{1}{\alpha^2}||\oW^TU||^2_F+||V^T\oZ||^2_F)\, .
\end{eqnarray*}
In order proceed, 
consider Eq.~(\ref{eq:FiniteDiff}).
Since by tangency condition $U^TL=0$, we have $U^T\oW = mR(\cos\Theta-1)R^T$
whence 
\begin{eqnarray}
||U^T\oW||_F = m ||\cos\theta-1|| = 
\frac{m}{2}\, ||4\sin^2(\theta/2)||\le \frac{m}{2}\, ||2\sin(\theta/2)||^2
\end{eqnarray}
(here $\theta = (\theta_1,\dots,\theta_r)$ is the vector containing 
the diagonal elements of $\Theta$). A similar calculation reveals that 
$||\oW||_F^2 = m||2\sin(\theta/2)||^2$ thus proving 
$||U^T\oW||_F^2\le ||\oW||_F^4/4\le Cm\delta^2||\oW||_F^2$.
The bound $||V^T\oZ||_F^2\le Cn\delta^2 ||\oZ||_F^2$ is proved in 
the same way, thus yielding
\begin{eqnarray}
E_4 & \le & Cn\eps\sqrt{\alpha}\sigma_{\rm max}(S)^2\delta^2 \, d(\Xm,\Um)^2\, .
\label{eq:E4}
\end{eqnarray}
By a similar calculation
\begin{eqnarray*}
E_5 & = & \frac{2\eps}{\sqrt{\alpha}}\Trace\{(S-\Sigma)S^T\oW^TU\}
 \le  \frac{2\eps}{\sqrt{\alpha}}\sigma_{\rm max}((S-\Sigma)S^T)
||\oW^TU||_F\\ 
&\le &n\eps\sqrt{\alpha} \sigma_{\rm max}(S) ||S-\Sigma||_F
d(\Um,\Xm)^2\, .
\end{eqnarray*}
and analogously 
\begin{eqnarray*}
E_6 \le n\eps\sqrt{\alpha} \sigma_{\rm max}(S) ||S-\Sigma||_F
d(\Um,\Xm)^2\, .
\end{eqnarray*}
Combining these estimates, and using $A^{2}\ge \E\{A^2\}/2$, 
we get
\begin{eqnarray*}
\frac{1}{n\eps}A^2 &\ge& C_1\sqrt{\alpha}||S-\Sigma||_F^2+
C_1\sqrt{\alpha}\sigma_{\rm min}(S)^2 d(\Um,\Xm)^2 \\
& & \;\;\;\; -\,
C_2\sqrt{\alpha}\sigma_{\rm max}(S)^2\delta^2d(\Um,\Xm)^2 - C_2\sqrt{\alpha}\sigma_{\rm max}(S)\,
||S-\Sigma||_F\, d(\Um,\Xm)^2
\end{eqnarray*}
for some numerical constants $C_1$, $C_2>0$. 
Using the bounds $\sigma_{\rm min}(S)^2\ge \Sigma_{\rm min}^2/2-
||S-\Sigma||_F^2$, $\sigma_{\rm max}(S)^2\le 2\Sigma_{\rm max}^2+
2\, ||S-\Sigma||_F^2$, and the assumption  
$d(\Xm,\Um)\le \delta$ for $\delta \leq \Sigma_{\rm min}/C_0\Sigma_{\rm max}$, we get 
the claim (\ref{eq:SloppyBound}).

\vspace{0.1cm}

We are now left with the task of proving
the upper bound in Eq.~(\ref{eq:Quadratic}).
We can set $\Sigma=S$, thus obtaining
\begin{eqnarray*}
F(X,Y) & \le &
\frac{1}{2}\sum_{(i,j)\in E}(U\Sigma\oZ^T+\oW \Sigma V^T+\oW \Sigma\oZ^T)_{ij}^2\\
&\le & \hA^2 + \hB^2\, ,
\end{eqnarray*}
where we defined
\begin{eqnarray*}
\hA^2 & \equiv &\sum_{(i,j)\in E}(U\Sigma\oZ^T+\oW \Sigma V^T)_{ij}^2\, ,\\
\hB^2 &\equiv &  \sum_{(i,j)\in E}(\oW \Sigma\oZ^T)_{ij}^2\, .
\end{eqnarray*}
Bounds for these two quantities are derived as for $A^2$ and $B^2$. 
More precisely, by Theorem 4.1 in \cite{CaR08}, we 
have $\hA^2\le 2\E\{ \hA^2\}$ with probability at least $1-1/n^5$ and  
\begin{eqnarray*}
\E\{ \hA^2\} & =& \frac{\eps}{n\sqrt{\alpha}}||\oW \Sigma V^T + U\Sigma\oZ^T||_F^2\\
& = &  \frac{2\eps}{n\sqrt{\alpha}} ||\oW \Sigma V^T||^2_F + \frac{2\eps}{n\sqrt{\alpha}} ||U\Sigma\oZ^T||^2_F \\
& \le & 2 \sqrt{\alpha} n\eps \Sigma_{\rm max}^2 \Big(\frac{1}{m}||\oW||^2_F + \frac{1}{n}||\oZ||^2_F\Big)
\le 4 \sqrt{\alpha} n\eps \Sigma_{\rm max}^2  d(\Xm,\Um)^2\, .
\end{eqnarray*}
$\hB^2$ is bounded similar to $B^2$ and we get,
\begin{eqnarray*}
\hB^2 \le C' \sqrt{\alpha}  n\eps\Sigma_{\rm max}^2d(\Um,\Xm)^2\, .
\end{eqnarray*}
\end{proof}
%
%
\section{Proof of Lemma \ref{lemma:Gradient}}
\label{sec:LemmaGrad} 

As in the proof of Lemma \ref{lemma:Quadratic},
see Section \ref{sec:PreliminariesProof}, we let
$t\mapsto \Xm(t) = (X(t),Y(t))$ be the geodesic 
starting at $\Xm(0) = \Um$ with velocity $\dot{\Xm}(0) = \Wm
=(W,Z)\in\Tang_{\Um}$. We also define
$\Xm = \Xm(1) = (X,Y)$ with $X= U+\oW$ and $Y=V+\oZ$. 
Let $\Whm = \dot{\Xm}(1) = (\hW,\hZ)$ 
be its velocity when passing through $\Xm$. An explicit
expression is  obtained in terms of the singular value decomposition 
of $W$ and $Z$. If we let $W= L\Theta R^T$, and differentiate
Eq.~(\ref{eq:Geodesic}) with respect to $t$ at $t=1$, we obtain
\begin{eqnarray}
\hW = -UR\Theta\sin\Theta\, R^T + L\Theta\cos\Theta\, R^T\, .\label{eq:Tangent1}
\end{eqnarray}
An analogous expression holds for $\hZ$. Since $L^TU=0$,
we have $||\hW||_F^2 = m||\Theta\sin\Theta||_F^2+m||\Theta\cos\Theta||_F^2=
m||\theta||^2$. Hence\footnote{Indeed this conclusion
could have been reached immediately, since
$t\mapsto \Xm(t)$ is a geodesic parametrized proportionally
to the arclength in th interval $t\in[0,1]$.} 
\begin{eqnarray}
 \frac{1}{m}||\hW||_F^2 + \frac{1}{n}||\hZ||_F^2 = d(\Xm,\Um)^2 . \label{eq:DistanceBoundBB}
\end{eqnarray}
In order to prove the thesis, it is therefore sufficient to 
lower bound $\<\grad \tF(\Xm),\Whm\>$.
In the following we will indeed show that 
\begin{eqnarray*}
\<\grad F(\Xm),\Whm\> \ge C\sqrt{\alpha}\,n\eps\Sigma_{\rm min}^2\, d(\Xm,\Um)^2\, , 
\end{eqnarray*}
and $\<\grad G(\Xm),\Whm\>\ge 0$, which together imply the thesis 
by Cauchy-Schwarz inequality.

Let us prove a few preliminary estimates.
\begin{remark}
With the above definitions,
$\Whm\in \Co((11/2)\pi^2\mu_0)$. 
\end{remark}
\begin{proof}
Since $\Theta = \diag(\theta_1,\dots,\theta_r)$
with $|\theta_i|\le \pi/2$, we get
\begin{eqnarray}
||\hW^{(i)}||^2\le 2||\Theta\sin\Theta R^TU^{(i)}||^2+
2||\Theta\cos\Theta L^{(i)}||^2\le
\frac{\pi^2}{2}||U^{(i)}||^2+
2||W^{(i)}||^2\, .
\end{eqnarray}
By assumption we have $||U^{(i)}||^2\le \mu_0 r$ and by Remark \ref{rem:IncoherenceBound} 
we have $||W^{(i)}||^2\le 5\pi^2\mu_0r/2$.
\end{proof}

One important fact that we will use is that $\hW$ is well approximated
by $W$ or by $\oW$, and $\hZ$ is well approximated by $Z$ or by $\oZ$. 
Using Eqs.~(\ref{eq:FiniteDiff}) and (\ref{eq:Tangent1}) 
we get
\begin{eqnarray}
||\hW||_F^2 & = & ||W||_F^2 = m||\theta||^2\, ,\label{eq:hWSize}\\
||\oW||_F^2 & = & m||2\sin\theta/2||^2\, ,\label{eq:oWSize}\\
\<\hW,\oW\> & = & m\sum_{a=1}^r\theta_a\sin\theta_a\, ,\\
\<\hW,W\> & = & m\sum_{a=1}^r\theta_a^2\cos\theta_a\, ,
\end{eqnarray}
and therefore
\begin{eqnarray}
||\hW-\oW||_F^2 & = & m\sum_{a=1}^r[(2\sin(\theta_a/2))^2+\theta_a^2-2\theta_a\sin\theta_a]\\
&\le & m\sum_{a=1}^r(\theta_a-2\sin(\theta_a/2))^2
\le \frac{m}{24^2} ||\theta||^4\le \frac{m}{24^2} d(\Um,\Xm)^4\, .
\label{eq:EstimateOWHW}
\end{eqnarray}
Analogously
\begin{eqnarray}
||\hW-W||_F^2 & = &n\sum_{a=1}^r[2\theta_a^2-2
\theta_a^2\cos\theta_a]
\le  m\, ||\theta||^4\le m\, d(\Um,\Xm)^4 \;,
\end{eqnarray}
where we used the inequlity $2(1-\cos x)\leq x^2$.
The last inequality implies in particular
\begin{eqnarray}
||U^T\hW||_F = ||U^T(W-\hW)||_F\le m d(\Um,\Xm)^2\, .\label{eq:UTW}
\end{eqnarray}
Similar bounds hold of course for $Z,\hZ,\oZ$ (for instance we have 
$||V^T\hZ||_F \le n d(\Um,\Xm)^2$). Finally, we shall use repeatedly the
fact that $||S-\Sigma||_F^2\le C\Sigma_{\rm max}^2d(\Xm,\Um)^2$, which follows from 
Lemma~\ref{lemma:Quadratic}. This in turns implies
\begin{eqnarray}
\sigma_{\rm max}(S)&\le &\Sigma_{\rm max}+C\, \Sigma_{\rm max} 
 d(\Xm,\Um)\le 2\,\Sigma_{\rm max}\, ,
\label{eq:EValueBound1}\\
\sigma_{\rm min}(S)&\ge &\Sigma_{\rm min}-C\,  \Sigma_{\rm max} 
d(\Xm,\Um)\ge \frac{1}{2}\, \Sigma_{\rm min}\, ,
\label{eq:EValueBound2}
\end{eqnarray}
where we used the hypothesis $d(\Xm,\Um)\le \delta=\Sigma_{\rm min}/
C_0\Sigma_{\rm max}$.
%
%
\subsection{Lower bound on $\grad F(\Xm) $}

Recalling that $\cP_E$ is the projector defined in 
Eq.~(\ref{eq:ProjectorDef}), and using the expression
(\ref{eq:gradX}), (\ref{eq:gradY}),  for the gradient, we have
\begin{align}
\<\grad &F(\Xm),\Whm\>  =  \<\cP_E(XSY^T-M),(XS\hZ^T+\hW SY^T)\>
\nonumber \\
& =  \<\cP_E(U(S-\Sigma)V^T+US\oZ^T+\oW SV^T+
\oW S\oZ^T),(US\hZ^T+\hW SV^T+ \oW S\hZ^T+  \hW S\oZ^T)\>\nonumber\\
&\ge A-B_1-B_2-B_3\label{eq:VariousTerms}
\end{align}
where we defined
\begin{eqnarray*}
A & = & \<\cP_E(US\oZ^T+\oW SV^T),(US\hZ^T+\hW SV^T)\>\, ,\\
B_1 & = & |\<\cP_E(US\oZ^T+\oW SV^T), ( \oW S\hZ^T+   \hW S\oZ^T)\>|\, ,\\
B_2 & = &|\<\cP_E(U(S-\Sigma)V^T+\oW S\oZ^T),(US\hZ^T+\hW SV^T)\>|\, ,\\
B_3 &= & |\< \cP_E(U(S-\Sigma)V^T+\oW S\oZ^T),( \oW S\hZ^T+   \hW S\oZ^T)\>| 
\, .
\end{eqnarray*}
At this point the proof becomes very similar to the one in the previous section
and consists in lower bounding $A$ and upper bounding $B_1$, $B_2$, $B_3$.

\subsubsection{Lower bound on $A$} \label{lbonA}
Using Theorem 4.1 in \cite{CaR08}
we obtain, with probability larger than $1-1/n^5$.
\begin{eqnarray*}
A& \ge &\frac{\eps}{2\sqrt{mn}}\<(US\oZ^T+\oW SV^T),(US\hZ^T+\hW SV^T)\>\\
&& 
\ge \frac{1}{2}\, A_0-\frac{1}{2}\, B_0
\end{eqnarray*}
where
\begin{eqnarray*}
A_0 &=& \frac{\eps}{2\sqrt{mn}}||US\oZ^T+\oW SV^T||_F^2\, ,\\
B_0 &=& \frac{\eps}{2\sqrt{mn}} ||US\oZ^T+\oW SV^T||_F ||US(\oZ-\hZ)^T+(\oW-\hW) SV^T||_F\, .
\end{eqnarray*}
The  term $A_0$ is lower bounded analogously to $\E\{A^2\}$
in the proof of Lemma \ref{lemma:Quadratic}, see
Eqs.~(\ref{eq:E2E3}) and (\ref{eq:E4}):
\begin{eqnarray*}
A_0 &=& \frac{\eps}{2\sqrt{mn}}||US\oZ^T+\oW SV^T||_F^2\, \\
& = & \frac{\eps}{2\sqrt{mn}}||US\oZ^T||_F^2 +
 \frac{\eps}{2\sqrt{mn}}||\oW SV^T||_F^2 + 
\frac{\eps}{2\sqrt{mn}}\<US\oZ^T,\oW SV^T\>\\
&\ge & Cn\eps (\sqrt\alpha\,\sigma_{\rm min}(S)^2 - \sqrt{\alpha}\delta^2\sigma_{\rm max}(S)^2)
d(\Xm,\Um)^2\ge C\sqrt{\alpha}\,n\eps\,\Sigma_{\rm min}^2\, d(\Xm,\Um)^2\, ,
\end{eqnarray*}
where we used the bounds (\ref{eq:EValueBound1}), 
(\ref{eq:EValueBound2}) and the hypothesis
$d(\Xm,\Um)\le \delta=\Sigma_{\rm min}/C_0\Sigma_{\rm max}$.

As for the second term we notice that
\begin{eqnarray}
\frac{B_0^2}{A_0} &\le &
n\eps\sqrt{\alpha}\Big( \frac{1}{m}||S(\oW-\hW)||_F^2+\frac{1}{n}||S^T(\oZ-\hZ)||_F^2 \Big)\\
&\le& n\eps\sqrt{\alpha}\,\sigma_{\rm max}(S)^2\Big( \frac{1}{m}|| \oW-\hW||_F^2 + \frac{1}{n}|| \oZ-\hZ||_F^2 \Big)
\le Cn\eps\sqrt{\alpha}\,\Sigma_{\rm max}^2 d(\Xm,\Um)^4\, ,
\end{eqnarray}
where, in the last step, we used the estimate 
(\ref{eq:EstimateOWHW}) and the analogous one for $||\oZ-\hZ||_F^2$.
Therefore for $d(\Xm,\Um)\le \delta\le \Sigma_{\rm min}/C_0\Sigma_{\rm max}$ 
and $C_0$ large enough $A_0>2B_0$, whence 
\begin{eqnarray}
A\ge  C\sqrt{\alpha}\,n\eps\,\Sigma_{\rm min}^2\, d(\Xm,\Um)^2\, . 
\end{eqnarray}

\subsubsection{Upper bound on $B_1$}\label{ubonB1}
We begin by noting that $B_1$ can be bounded above by the sum of four terms of the form $B_1' = |\<\cP_E(US\oZ^T),\oW S\hZ^T\>|$. We show that $B_1' < A/100$. The other terms are bounded similarly.

Using Remark \ref{rem:AprioriBound}, we have

\begin{eqnarray*}
||\cP_E(\oW S \hZ^T)||_F^2 
& \le & C \frac{\eps}{\sqrt{mn}}||\oW||_F^2 ||S\hZ^T||_F^2 + C' \sqrt{\eps} \mu_0 r \sqrt{\alpha} \Smax ||\oW||_F ||S\hZ||_F \\
& \le & 
2C \frac{\eps}{\sqrt{mn}}||\oW||_F^2 ||S\oZ^T||_F^2 +
2C \frac{\eps}{\sqrt{mn}}||\oW||_F^2 ||S(\hZ-\oZ)^T||_F^2 \\
& & +
C' \sqrt{\eps} \mu_0 r \sqrt{\alpha} \Smax ||\oW||_F ||S \oZ||_F  +  
C' \sqrt{\eps} \mu_0 r \sqrt{\alpha}\Smax ||\oW||_F ||S (\hZ-\oZ)||_F \\
& \le & C'' A \left( \delta^2 + \frac{\mu_0 r \sqrt{\alpha} \Smax}{\sqrt{\eps} \Smin} \right) 
\end{eqnarray*}

where we have used $\frac{\eps m}{\sqrt{mn}}||S\oZ^T||_F^2 \le 3A_0 \le 12A$ from Section \ref{lbonA}. Therefore we have,

\begin{eqnarray*}
B_1'^2 
& \le & ||\cP_E(US\oZ^T)||_F^2 ||\cP_E(\oW S \hZ^T)||_F^2\\
& \le & C \frac{\eps}{\sqrt{mn}} ||US\oZ^T||_F^2 A \left( \delta^2 + \frac{\mu_0 r \sqrt{\alpha} \Smax}{\sqrt{\eps} \Smin} \right) \\
& \le & C' A^2 \left( \delta^2 + \frac{\mu_0 r \sqrt{\alpha} \Smax}{\sqrt{\eps} \Smin} \right) \\
\end{eqnarray*}

The thesis follows for $\delta$ and $\eps$ as in the hypothesis.

\subsubsection{Upper bound on $B_2$} 
We have
\begin{eqnarray*}
B_2&\le &|\<\cP_E(US\hZ^T+\hW SV^T),\oW S\oZ^T\>|+
 |\<\cP_E(US\hZ^T),U(S-\Sigma)V^T\>|\\
 &&\;\;\;\; +|\<\cP_E(\hW SV^T),U(S-\Sigma)V^T\>|\\
&\equiv& B_2'+B_2''+B_2'''\, .
\end{eqnarray*} 
We claim that each of these three terms is smaller than $A/30$,
whence $B_2\le A/10$.

The upper bound on $B_2'$ is obtained similarly to 
the one on $B_1$ to get $B_2'\le  A/30$.

Consider now $B_2''$. By Theorem 4.1 in \cite{CaR08}, 
\begin{eqnarray*}
B_2'' 
&\leq&  \frac{\eps}{\sqrt{mn}} |\<U(S-\Sigma)V^T,US\hZ^T\>|+ C \frac{\eps}{\sqrt{mn}} \sqrt{\frac{\mu_0 n r \alpha \log n}{n \sqrt{\alpha} \eps}} ||U(S-\Sigma)V^T||_F ||US\hZ||_F
\end{eqnarray*}

To bound the second term, observe 
\begin{eqnarray*}
||US\hZ^T||_F & \le & ||US\oZ^T||_F + ||US(\hZ - \oZ)^T||_F\\
& \le & ||US\oZ^T||_F + \Smax \sqrt{mn} \, d(\Xm,\Um)^2
\end{eqnarray*}

Also, $\frac{\eps}{\sqrt{mn}}||US\oZ^T||_F^2 \le 3A_0 \le 12A$ from Section \ref{lbonA}. Combining these, we have that the second term in $B_2''$ is smaller than $A/60$ for $\eps$ as in the hypothesis. 

To bound the first term in $B_2''$,  
\begin{eqnarray*}
|\<U(S-\Sigma)V^T,US\hZ^T\>|
&=&|\<U(S-\Sigma)(Y-V)^T,US\hZ^T\>| \\
&\le& ||U(S-\Sigma)\oZ^T||_F ||US\oZ^T||_F +
||U(S-\Sigma)\oZ^T||_F ||US(\oZ - \hZ||_F
\end{eqnarray*}

Therefore
\begin{eqnarray*}
B_2'' 
&\le&  \frac{\eps}{\sqrt{mn}} ||U(S-\Sigma)\oZ||_F||US\oZ||_F+ \eps n\sqrt{\alpha} \Smax^2 d(\Xm,\Um)^4 + A/60 \\
&\le& \frac{\eps}{\sqrt{mn}} ||U(S-\Sigma)\oZ||_F||US\oZ||_F + A/40
\end{eqnarray*}
for $d(\Xm,\Um) \le \delta$ as in the hypothesis.

We are now left with upper bounding $\tB_2'' \equiv \frac{\eps}{\sqrt{mn}} ||U(S-\Sigma)\oZ||_F||US\oZ||_F $.

\begin{eqnarray*}
\tB_2''^2 
&\le& \left( \frac{\eps}{\sqrt{mn}} ||U(S-\Sigma)\oZ^T||_F^2 \right)
\left( \frac{\eps}{\sqrt{mn}} ||US\oZ^T||_F^2\right) \\
&\le& \left( \eps n \sqrt{\alpha} \Smax^2 d(\Xm,\Um)^4 \right)
\left( \frac{\eps}{\sqrt{mn}} ||US\oZ^T||_F^2\right) 
\end{eqnarray*}

Also from the lower bound on A, we have,$\frac{\eps}{\sqrt{mn}} ||US\oZ^T||_F^2 \le 3A_0 \le 12A$. Using $d(\Xm,\Um) \le \delta$, we have $\tB_2'' \le A/120$ for $\delta$ as in the hypothesis. This proves the desired result. The bound on $B_2'''$ is calculated analogously.

\subsubsection{Upper bound on $B_3$}

Finally for the last term it is sufficient to use
a crude bound
\begin{eqnarray*}
B_3\le 4\Big(||\cP_E(\oW S\hZ^T)||_F+ ||\cP_E(\hW S\oZ^T)||_F\Big)
\Big(||\cP_E(U(S-\Sigma)V^T)||_F+||\cP_E(\oW S\oZ^T)||_F\Big) \, ,
\end{eqnarray*}
The terms of the form $||\cP_E(\oW S \hZ^T)||_F $ are all estimated as in Section \ref{ubonB1}. Also, by Theorem 4.1 of \cite{CaR08}

\begin{eqnarray*}
||\cP_E(U (S-\Sigma) V^T)||_F
&\le& C \frac{\eps}{\sqrt{mn}} ||U (S-\Sigma) V^T||_F^2\\
&\le& C n\eps \sqrt{\alpha} \Smax^2 d(\Xm,\Um)^2
\end{eqnarray*}

Combining these estimates with the $\delta$ and the $\eps$ in the hypothesis, we get $B_3 \le A/10$

%
%
\subsection{Lower bound on $\grad G(\Xm) $}

By the definition of $G$ in Eq.~(\ref{eq:Gdef}), we have
\begin{eqnarray}
\<\grad G(\Xm),\Whm\> = \frac{1}{\mu_0r}
\sum_{i=1}^mG_1'\left(
\frac{||X^{(i)}||^2}{3\mu_0 r}\right)\<X^{(i)},\hW^{(i)}\>+
\frac{1}{\mu_0r}
\sum_{j=1}^nG_1'\left(
\frac{||Y^{(i)}||^2}{3\mu_0 r}\right)\<Y^{(i)},\hZ^{(i)}\>\, .
\end{eqnarray}
It is therefore sufficient to show that if
$||X^{(i)}||^2>3\mu_0r$, then $\<X^{(i)},\hW^{(i)}\> > 0$,
and if
$||Y^{(j)}||^2>3\mu_0r$, then $\<Y^{(j)},\hZ^{(j)}\> > 0$.
We will just consider the first statement, the second being 
completely symmetrical.

From the explicit expressions (\ref{eq:FiniteDiff}) and 
(\ref{eq:Tangent1}) we get
\begin{eqnarray}
X^{(i)} & = &R\left\{\cos\Theta\, R^TU^{(i)}+\sin\Theta\, L^{(i)}\right\}\, ,
\label{eq:X(i)}\\
\hW^{(i)} & = & R\left\{\Theta\cos\Theta L^{(i)}-\Theta\sin\Theta R^TU^{(i)}
\right\}\, .\label{eq:W(i)}
\end{eqnarray}
From the first expression it follows that 
\begin{eqnarray*}
||\sin\Theta\, L^{(i)}||^2\le ||X^{(i)}||^2 + ||\cos\Theta\, R^TU^{(i)}||^2 \le 5\,\mu_0r\, .
\end{eqnarray*}
On the other hand, 
by taking the difference of Eqs.~(\ref{eq:X(i)}) and (\ref{eq:W(i)})
we have 
\begin{eqnarray*}
||X^{(i)}-\hW^{(i)}|| &\le &  ||(\sin\Theta -\Theta\cos\Theta)L^{(i)}||
 +||(\cos\Theta+\Theta\sin\Theta)R^{T}U^{(i)}||\\
&\le &
 \max_i(\theta_i^2)||\sin\Theta L^{(i)}||+\frac{\pi}{2}||U^{(i)}||
 \le \delta \sqrt{4\mu_0r}+\frac{\pi}{2}\sqrt{\mu_0 r}\, .
\end{eqnarray*}
where we used the inequality $(\sin\omega-\omega\cos\omega)\le \omega^2
\sin\omega$ valid for $\omega\in[0,\pi/2]$. For $\delta$
small enough we have therefore $||X^{(i)}-\hW^{(i)}||\le (99/100)\sqrt{3\mu_0r}$.
To conclude, for $||X^{(i)}||\ge 3\mu_0r$
\begin{eqnarray*}
\<X^{(i)},\hW^{(i)}\> \ge ||X^{(i)}||^2- ||X^{(i)}||\,
 ||X^{(i)}-\hW^{(i)}||\ge ||X^{(i)}||(\sqrt{3\mu_0 r}- (99/100)\sqrt{3\mu_0r})
\ge 0\, .
\end{eqnarray*}

\endproof

\section*{Acknowledgements}

We thank Emmanuel Cand\'es and
Benjamin Recht for stimulating discussions on the subject
of this paper. This work was partially supported by
a Terman fellowship and an NSF CAREER award (CCF-0743978).

%
%
\appendix

%
%
\section{Proof of Remark~\ref{remark:sizetrim}} \label{app:sizetrim}

The proof is a generalization of analogous result in \cite{FeO05}, 
which is proved to hold only with probability larger than $1-e^{-C\eps}$. 
The stronger statement quoted here can be proved using 
concentration of measure inequalities.

First, we apply Chernoff bound to the event 
$\Big\{|\overline{\cA}_l| > \max\{e^{-C_1\eps} m,C_2\alpha\}\Big\}$. 
In the case of large $\eps$, when $\eps>3\sqrt\alpha\log n$, 
we have  $\prob \left\{|\overline{\cA}_l| > C_2\alpha \right\} \leq 1/2n^4$,
for $C_2\geq\max\{e,26/\alpha\}$.
In the case of small $\eps$, when $\eps\leq3\sqrt\alpha\log n$,
we have $\prob \left\{|\overline{\cA}_l| > \max\{e^{-C_1\eps} m,C_2\alpha\}\right\} \leq 1/2n^4$,
for $C_1\leq1/600\sqrt\alpha$ and $C_2\geq130$. 
Here we made a moderate assumption of $\eps \geq 3\sqrt\alpha$,
which is typically in the region of interest. 

Analogously, we can prove that 
$\prob\left\{|\overline{\cA}_r| > \max\{e^{-C_1\eps} n,C_2\}\right\}\leq 1/2n^4$ 
, which finishes the proof of Remark~\ref{remark:sizetrim}.

\section{Proof of Remark~\ref{rem:lightexpectation}} \label{app:lightexpectation}

The expectation of the contribution of light couples,
when each edge is independently revealed with probability $\eps/\sqrt{mn}$, is
\begin{eqnarray*}
\E[Z] &=& \frac{\eps}{\sqrt{mn}} \left( \sum_{(i,j)\in L}x_i M^A_{ij} y_j - x^TMy \right) \;,
\end{eqnarray*}
where we define $M^A$ by setting to zero the rows of $M$ whose index is not in $A_l$ and the columns of $M$ whose index is not in $A_r$.

In order to bound $\sum_{ L}x_i M^A_{ij} y_j - x^TMy$,
we write,

\begin{eqnarray}
\left| \sum_{(i,j)\in L}x_i M^A_{ij} y_j - x^TMy \right| &=& \left| x^T\Big(M^A - M\Big)y - \sum_{(i,j)\in \bL}x_i M^A_{ij} y_j \right| \nonumber\\
&\leq& \left| x^T\Big(M^A - M\Big)y \right| + \left| \sum_{(i,j)\in \bL}x_i M^A_{ij} y_j \right| \;.\nonumber
\end{eqnarray}

Note that $|(M^A - M)_{ij}|$ is non-zero only if 
$i \notin A_l$ or $j \notin A_r$, in which case $|(M^A - M)_{ij}| \leq \Mmax$.
Also, by Remark \ref{remark:sizetrim}, there exists 
$\delta = \max\{e^{-C_1\eps},C_2/n\}$ such that $|{i:i\notin A_l}| \le \delta m$ and $|{j:j\notin A_r}| \le \delta n$. Denoting by $\ind(\,\cdot\,)$ the indicator function, we have
\begin{eqnarray*}
\left| x^T\Big(M^A - M\Big)y \right| &\le& \sum_{ij} \big|x_i\big| \big|y_j\big| \Big( \ind({i \notin A_l}) + \ind({j \notin A_r}) \Big) \Mmax \nonumber\\
 &=& \left( \sum_{i}\big|x_i\big|\ind({i \notin A_l})\sum_j\big|y_j\big| + \sum_{j}\big|y_j\big|\ind({j \notin A_r})\sum_i\big|x_i\big| \right) \Mmax \nonumber\\
& \le & \left( \sqrt{\delta m}\sqrt{n} + \sqrt{\delta n}\sqrt{m} \right) \Mmax \\
& \le& \Mmax\sqrt\frac{mn}{\eps} \;.
\end{eqnarray*}
for $\delta \le \frac{1}{4\eps}$.
We can bound the second term as follows 
\begin{eqnarray*}
\left|\sum_{(i,j)\in \bL}x_i M^A_{ij} y_j\right| &\leq& \sum_{(i,j)\in \bL}\frac{\left|x_i M^A_{ij} y_j\right|^2}{\left|x_i M^A_{ij} y_j\right|^{\phantom{2}}} \\
 &\leq& \frac{1}{\Mmax}\sqrt{\frac{mn}{\eps}} \sum_{(i,j)\in \bL}{\left|x_i M^A_{ij} y_j\right|^2} \\
 &\leq& \frac{1}{\Mmax}\sqrt{\frac{mn}{\eps}} \sum_{i\in [m],j\in[n]}{\left|x_i M^A_{ij} y_j\right|^2} \\
 &\leq& \Mmax\sqrt{\frac{mn}{\eps}} \;,
\end{eqnarray*}
where the second inequality follows from the definition of heavy couples.

Hence, summing up the two contributions, we get
\begin{eqnarray*}
\left| \E \left[ Z \right] \right| & \leq & 2\Mmax \sqrt{\eps}\;.
\end{eqnarray*}

%
%
\section{Proof of Remark~\ref{remark:HeavyPart}} \label{app:HeavyPart}
We can associate to the matrix $Q$ a bipartite graph 
 $\cG= ([m],[n],\cE)$.
The proof is similar to the one in \cite{FKS89,FeO05} and is
based on two properties of the graph $\cG$:
\begin{enumerate}
\item \emph{Bounded degree.}
 The graph $\cG$ has maximum degree bounded by
a constant times the average degree:  
\begin{eqnarray}
\deg(i) &\leq& \frac{2\eps}{\sqrt{\alpha}}\;,\\
\deg(j) &\leq& 2\eps\sqrt{\alpha}\;, \label{eq:boundeddegree}
\end{eqnarray}
for all $i\in[m]$ and $j\in[n]$.
\item \emph{Discrepancy.}
We say that $\cG$  (equivalently, the adjacency matrix $Q$)
has the discrepancy property if, for any 
$A \subseteq [m]$ and $B \subseteq [n]$, one of the following is true: 
\begin{eqnarray}
&1.&\frac{e(A,B)}{\mu(A,B)} \leq \xi_1 \;,\label{eq:discrepancy01}\\
&2.&e(A,B)\ln\Big(\frac{e(A,B)}{\mu(A,B)}\Big) \leq \xi_2 \max\{|A|/\sqrt{\alpha},|B|\sqrt{\alpha}\} \ln\Big(\frac{\sqrt{mn}}{\max\{|A|/\sqrt{\alpha},|B|\sqrt{\alpha}\}}\Big)\;\label{eq:discrepancy02}.
\end{eqnarray}
for two numerical constants $\xi_1$, $\xi_2$ (independent of $n$ and $\eps$).
Here $e(A,B)$ denotes the number of edges between $A$ and $B$ 
and $\mu(A,B)=|A||B||E|/mn$ denotes the average number of edges 
between $A$ and $B$ before trimming. 
\end{enumerate}
We will prove, later in this section, 
that the discrepancy property holds with high probability. 

Let us partition row and column indices with respect to the 
value of $x_u$ and $y_v$:
\begin{eqnarray*}
A_i&=&\{u\in [m]\;:\; \frac{\Delta}{\sqrt{m}}2^{i-1}\leq |x_u| < \frac{\Delta}{\sqrt{m}}2^{i}\}\;,\\ 
B_j&=&\{v\in [n] \;:\; \frac{\Delta}{\sqrt{n}}2^{j-1}\leq |y_v| < \frac{\Delta}{\sqrt{n}}2^{j}\}\;,
\end{eqnarray*}
for $i\in\{1,2,\dots,\lceil{\ln{(\sqrt{m}/\Delta)}/\ln{2}}\rceil\}$, and 
$j\in\{1,2,\dots,\lceil{\ln{(\sqrt{n}/\Delta)}/\ln{2}}\rceil\}$, and
we denote the size of subsets $A_i$ and $B_j$ by $a_i$ and $b_j$ respectively. 
Furthermore, we define $e_{i,j}$ to be the number of edges between 
two subsets $A_i$ and $B_j$, and we let $\mu_{i,j}=a_i b_j(\eps/\sqrt{mn})$.
Notice that all indices $u$ of non zero $x_u$ fall into one of the 
subsets $A_i$'s defined above, since, by discretization, the smallest non-zero element 
of $x\in T_m$ in absolute value is at least $\Delta/\sqrt{m}$. 
The same applies for the entries of $y\in T_n$.

By grouping the summation into $A_i$'s and $B_j$'s, we get
\begin{eqnarray*}
\sum_{\substack{(u,v):\\ |x_uy_v|\ge \frac{C\sqrt{\eps}}{\sqrt{mn}}}}Q_{uv}|x_uy_v| &\leq& \sum_{(i,j):2^{i+j}\geq \frac{4C\sqrt{\eps}}{\Delta^2}}{e_{i,j}\frac{\Delta 2^i}{\sqrt{m}}\frac{\Delta 2^j}{\sqrt{n}}} \\
&=& \Delta^2\sum{a_ib_j\frac{\eps}{\sqrt{mn}}\frac{e_{i,j}}{\mu_{i,j}}\frac{2^i}{\sqrt{m}}\frac{2^j}{\sqrt{n}}} \\
&=& \Delta^2\sqrt{\eps}\sum{\underbrace{a_i\frac{2^{2i}}{{m}}}_{\alpha_i} \underbrace{b_j\frac{2^{2j}}{{n}}}_{\beta_j} \underbrace{\frac{e_{i,j}\sqrt{\eps}}{\mu_{i,j}2^{i+j}}}_{\sigma_{i,j}} } .\\
\end{eqnarray*}
Note that, by definition, we have
\begin{eqnarray}
 \sum_i{\alpha_i} \leq  4||x||^2/\Delta^2 \;,\label{eq:sumalpha}\\
 \sum_i{\beta_i} \leq   4||y||^2/\Delta^2 \; \label{eq:sumbeta}.
\end{eqnarray}
We are now left with task of bounding $\sum \alpha_i\beta_j\sigma_{i,j}$, 
for $Q$ that satisfies bounded degree property and discrepancy property.

Define, 
\begin{eqnarray}
\cC_1 &\equiv& \left\{(i,j)\; : \; 2^{i+j}\geq \frac{4C\sqrt{\eps}}{\Delta^2}\text{ and }(A_i,B_j)\text{ satisfies }(\ref{eq:discrepancy01}) \right\}, \\
\cC_2 &\equiv& \left\{(i,j)\; : \; 2^{i+j}\geq \frac{4C\sqrt{\eps}}{\Delta^2}\text{ and }(A_i,B_j)\text{ satisfies }(\ref{eq:discrepancy02}) \right\}\setminus\cC_1 \; .
\end{eqnarray}
We need to show that $\sum_{(i,j)\in \cC_1\cup\cC_2} \alpha_i\beta_j\sigma_{i,j}$ is bounded.

For the terms in $\cC_1$ this bound is easy. 
Since summation is over pairs of indices $(i,j)$ such that
$2^{i+j}\geq \frac{4C\sqrt{\eps}}{\Delta^2}$, 
it follows from bounded degree property that $\sigma_{i,j}\leq \xi_1 \Delta^2/4C$.
By Eqs.~(\ref{eq:sumalpha}) and (\ref{eq:sumbeta}), we have
$\sum_{\cC_1}{\alpha_i\beta_j\sigma_{i,j}} \leq (\xi_1\Delta^2/4C) (2/\Delta)^4=O(1)$.

For the terms in $\cC_2$ the bound is more complicated.
We assume $a_i\leq\alpha b_j$ for simplicity and the other 
case can be treated in the same manner.
By change of notation the second discrepancy condition becomes
\begin{eqnarray}
e_{i,j}\log\left(\frac{e_{i,j}}{\mu_{i,j}}\right) \leq \xi_2\max\{a_i/\sqrt\alpha,b_j\sqrt\alpha\}\log\left(\frac{\sqrt{mn}}{\max\{a_i/\sqrt\alpha,b_j\sqrt\alpha \}}\right)\; .\label{eq:discrepancy2} 
\end{eqnarray}
We start by changing variables on both sides of Eq.~(\ref{eq:discrepancy2}).
   \begin{eqnarray*}
    \frac{e_{i,j}a_ib_j\eps}{\mu_{i,j}\sqrt{mn}}\log\left(\frac{e_{i,j}}{\mu_{i,j}}\right) &\leq& \xi_2 b_j\sqrt\alpha \log\left(\frac{2^{2j}}{ \beta_j}\right) \;.\\
   \end{eqnarray*}
   Now, multiply each side by $2^i/b_j\sqrt{\eps}2^j$ to get
   \begin{eqnarray}
    \sigma_{i,j}\alpha_i\log\left(\frac{e_{i,j}}{\mu_{i,j}}\right) 
    \leq \frac{\xi_2 2^i}{\sqrt{\eps}2^j}\left[\log(2^{2j})-\log\beta_j\right]\;.\label{eq:heavy_part}
   \end{eqnarray}
To achieve the desired bound, we partition the analysis into 5 cases:
\begin{enumerate}
\item $\sigma_{i,j} \leq 1$ : By Eqs.~(\ref{eq:sumalpha}) and (\ref{eq:sumbeta}), we have
   $\sum{\alpha_i\beta_j\sigma_{i,j}} \leq (2/\Delta)^4 = O(1)$.
\item $2^i>\sqrt{\eps}2^j$ : By the bounded degree property in Eq.~(\ref{eq:boundeddegree}), we have $e_{i,j} \leq a_i 2\eps/\sqrt{\alpha}$, 
   which implies that $e_{i,j}/\mu_{i,j}\leq 2n/b_j$.
   For a fixed $i$ we have, 
   $\sum_{j}{\beta_j\sigma_{i,j}\ind(2^i>\sqrt{\eps}2^j)} \leq 
   2\sqrt{\eps}\sum_{j}{2^{j-i}\ind(2^i>\sqrt{\eps}2^j)} \leq 4$. Then,
   $\sum{\alpha_i\beta_j\sigma_{i,j}} \leq {16}/{\Delta^2}=O(1)$.
\item $\log\left({e_{i,j}}/{\mu_{i,j}}\right)>\frac{1}{4}\left[\log(2^{2j})-\log\beta_j\right]$ :    From Eq.(\ref{eq:heavy_part}), it immediately follows that 
   $\sigma_{i,j}\alpha_i \leq \frac{4\xi_22^i}{\sqrt{\eps}2^j}$.
   Due to case 2, we can assume 
   $2^i \leq \sqrt{\eps}2^j$, which implies that 
   for a fixed j we have the following inequality :  
   $ \sum_{i}{\sigma_{i,j}\alpha_i} \leq
   4\xi_2\sum_{i}\frac{2^i}{\sqrt{\eps}2^j}\ind(2^i \leq \sqrt{\eps}2^j)
   \leq 8\xi_2$. Then it follows by Eq.~(\ref{eq:sumbeta}) that $\sum{\alpha_i\beta_j\sigma_{i,j}} \leq {32\xi_2}/{\Delta^2}=O(1)$.
\item $\log(2^{2j}) \geq -\log\beta_j$ : 
   Due to case 3, we can assume $\log\left({e_{i,j}}/{\mu_{i,j}}\right) \leq \frac{1}{4}\left[\log(2^{2j})-\log \beta_j \right]$, 
   which implies that $\log\left({e_{i,j}}/{\mu_{i,j}}\right) \leq \log(2^{j})$. 
   Further, since we are not in case 1, we can assume $1 < \sigma_{i,j} = {e_{i,j}\sqrt{\eps}}/{\mu_{i,j}2^{i+j}}$.
   Combining those two inequalities, we get $2^i \leq \sqrt{\eps}$.

   Since in defining $\cC_2$ we excluded $\cC_1$, if
   $(i,j)\in\cC_2$ then $\log\left({e_{i,j}}/{\mu_{i,j}}\right) \geq 1$. 
   Applying Eq.~(\ref{eq:heavy_part}) we get 
$\sigma_{i,j}\alpha_i \leq \sigma_{i,j}\alpha_i\log\left({e_{i,j}}/{\mu_{i,j}}\right) 
\leq ({\xi_22^{i-j}}/{\sqrt{\eps}})\left[\log(2^{2j})-\log\beta_j\right]  
\leq {4\xi_22^{i}}/{\sqrt{\eps}}$. 

   Combining above two results, it follows that 
$    \sum_{i}{\sigma_{i,j}\alpha_i}
     \leq 4\xi_2\sum_{i}{\frac{2^i}{\sqrt{\eps}}\ind(2^i \leq \sqrt{\eps})} \leq 8\xi_2\;$. 
   Then, we have the desired bound : $\sum{\alpha_i\beta_j\sigma_{i,j}} \leq \frac{32\xi_2}{\Delta^2}=O(1)$.
\item $\log(2^{2j}) < -\log\beta_j$ : 
  It follows, since we are not in case 3, that $\log\left({e_{i,j}}/{\mu_{i,j}}\right) \leq \frac{1}{4}\left[\log(2^{2j})-\log\beta_j\right] \leq -\log\beta_j$.
   Hence, ${e_{i,j}}/{\mu_{i,j}} \leq {1}/{\beta_j}$. 
   This implies that $\sigma_{i,j}={e_{i,j}\sqrt{\eps}}/{\mu_{i,j}2^{i+j}} \leq 
   {\sqrt{\eps}}/{\beta_j 2^{i+j}}$.
   Since the summation is over pairs of indices $(i,j)$ such that 
   $2^{i+j}\geq {4C\sqrt{\eps}}/{\Delta^2}$, we have
   $\sum_{j}{\sigma_{i,j}\beta_j}\leq\frac{\Delta^2}{2 C}$.
   Then it follows that $\sum{\alpha_i\beta_j\sigma_{i,j}} \leq \frac{2}{ C}=O(1)$.
\end{enumerate} 

  Analogous analysis for the set of indices $(i,j)$ such that $a_i>\alpha b_j$ will give us similar bounds.
  Summing up the results, we get that there exists a constant 
  $C'\le \frac{32}{\Delta^4} + \frac{4\xi_1}{C\Delta^2} +
       \frac{32}{\Delta^2} + \frac{128\xi_2}{\Delta^2} + \frac{4}{C} $, such that 
  \begin{eqnarray*}
   \sum_{(i,j):2^{i+j}\geq \frac{4C\sqrt{\eps}}{\Delta^2}}{{\alpha_i}{\beta_j}{\sigma_{i,j}}} \leq C'\;. 
  \end{eqnarray*}
  This finishes the proof of Remark~\ref{remark:HeavyPart}.
%
%
\begin{lemma}\label{lem:discrepancy}
The adjacency matrix $Q$ has discrepancy property with probability 
at least $1-{1}/{2n^3}$.
\end{lemma}
\begin{proof}
 The proof is a generalization of analogous result in \cite{FKS89,FeO05} 
which is proved to hold only with probability larger than $1-e^{-C\eps}$.
The stronger statement quoted here is a result of the observation
that, when we trim the graph the number of edges between any two 
subsets does not increase. 
Define $Q_0$ to be the adjacency matrix corresponding to 
original random matrix $M^E$ before trimming.
If the discrepancy assumption holds for $Q_0$,
then it also holds for $Q$, since $e^{Q}(A,B) \leq e^{Q_0}(A,B)$, 
for $A\subseteq[m]$ and $B\subseteq[n]$. 

Now we need to show that the desired property is satisfied for $Q_0$.
This is proved for the case of non-bipartite graph in Section 2.2.5 of \cite{FeO05},
and analogous analysis for bipartite graph shows that 
for all subsets $A\subseteq[m]$ and $B\subseteq[n]$, 
with probability at least $1-1/2(mn)^p$, the discrepancy condition holds 
with $\xi_1=2e$ and $\xi_2=(3p+12)(\alpha^{1/2}+\alpha^{-1/2})$.
Since we assume $\alpha\ge 1$, taking $p$ to be $3/2$ proves the desired thesis.
\end{proof}

%
%

%
\section{Proof of remarks \ref{remark:DistancesGrassmann}, \ref{rem:rescaling} and \ref{remark:Near}}
\label{app:Distance}

\begin{proof}(Remark \ref{remark:DistancesGrassmann}.)
Let $\theta = (\theta_1,\dots,\theta_p)$, $\theta_i\in[-\pi/2,\pi/2]$
 be the 
principal angles between the planes spanned by the columns of
$X_1$ and $X_2$. It is known that $d_{\rm c}(X_1,X_2)=
||2\sin(\theta/2)||_2$ and  $d_{\rm p}(X_1,X_2)=
||\sin \theta||_2$. The thesis follows from the elementary inequalities
\begin{eqnarray}
\frac{1}{\pi}\alpha\le\sqrt{2}\,\sin(\alpha/2)\le\sin\alpha\le 2\sin(\alpha/2)
\end{eqnarray}
valid for $\alpha\in [0,\pi/2]$.
\end{proof}

\begin{proof}(Remark \ref{rem:rescaling})
Given $X \in \reals^{n \times r}$, define $X'$ by
\begin{eqnarray*}
X'^{(i)} = \frac{X^{(i)}}{||X^{(i)}||} \min \left( ||X^{(i)}||, \sqrt{\mu_0 r} \right)
\end{eqnarray*}
for all $i \in [n]$.

\noindent Let $A$ be a matrix for extracting the ortho-normal basis of the columns of $X'$. That is $A \in \reals^{r \times r}$ such that $X'' = X'A$ and $X''^TX'' = n\id$. Without loss of generality, $A$ can be taken to be a symmetric matrix. In the following, let $\sigma_i = \sigma_i(A^{-1})$ for all $i \in [n]$. Note that by construction $d(U,X') \le d(U,X) \le \delta$. Hence there is a $Q_1 \in \Orth(r)$ such that, 
\begin{eqnarray}
||U - X'Q_1||_F^2 & \le & n \delta^2 \label{eq:rescalingmain}
\end{eqnarray}

\noindent We start by writing
\begin{eqnarray}
n A^{-T}A^{-1} =  X'^TX' = Q_1 (n\id - (U-X'Q_1)^TU + U^T(U-X'Q_1) + (U-X'Q_1)^T(U-X'Q_1) ) Q_1^T. \label{eq:rescalingmain1}
\end{eqnarray}

\noindent Using (\ref{eq:rescalingmain}), we have
\begin{eqnarray*}
||(U-X'Q_1)^TU||_F & \le & ||U||_2 ||(U-X'Q_1)||_F \\
& \le &  n \delta, \\
\end{eqnarray*}
and
\begin{eqnarray*}
||(U-X'Q_1)^T(U-X'Q_1)||_F & \le & n \delta^2. \\
\end{eqnarray*}

\noindent Therefore, using (\ref{eq:rescalingmain1})
\begin{eqnarray}
\sigma_1^2  & \le &  1 + 2\delta  + \delta^2 \label{eq:rescaling:1},\\
\sigma_r^2  & \ge &  1 - 2\delta  - \delta^2  \label{eq:rescaling:2}.
\end{eqnarray}
From (\ref{eq:rescaling:1}), (\ref{eq:rescaling:2}) and $\delta \le 1/16$, we get $\sigma_1 \le \sqrt{3}$ and $\sigma_r \ge 1/\sqrt{3}$. Since $||X''^{(i)}||^2 = ||X'^{(i)}A||^2 \le 3 \mu_0 r$ for all $i \in [n]$, we have that $X'' \in \Co(3 \mu_0)$.

\noindent We next prove that $d(X',X'') \le 3\delta$ which implies the thesis by triangular inequality. 
\begin{eqnarray*}
d(X',X'')^2 & = & \frac{1}{n} \min_{Q \in \Orth(r)} ||X' - X''Q||_F^2 \\
& \le & \frac{1}{n} ||X''A^{-1} - X''||_F^2 \\
& = & ||A^{-1} - \id||_F^2\\
&\le& ||(A^{-1} - \id)(A^{-1} + \id)||_F^2\\
& \le & ||A^{-T}A^{-1} - \id||_F^2 \\
&\le& 9 \delta^2
\end{eqnarray*}
where the last inequality is from (\ref{eq:rescalingmain1}).
\end{proof}

\begin{proof}(Remark \ref{remark:Near}.)
We start by observing that
\begin{eqnarray}
d_{\rm p}(V,Y) = \frac{1}{\sqrt{n}}\min_{A\in \reals^{r\times r}}
||V-YA||_F\, .\label{eq:IneqNorm1}
\end{eqnarray}
Indeed the minimization on the right hand side can be performed explicitly
(as $||V-YA||_F^2$ is a quadratic function of $A$) and the minimum is
achieved at $A= Y^TV/n$. The inequality follows by simple 
algebraic manipulations.

Take $A = S^TX^TU\Sigma^{-1}/m$. Then 
\begin{eqnarray}
||V-YA||_F &= &\sup_{B, ||B||_F\le 1}\<B,(V-YA)\>\\ 
&=& \sup_{B, ||B||_F\le 1}\<B^T,\frac{1}{m}\Sigma^{-1}U^T(U\Sigma V^T-XSY^T)\>\\
&=& \frac{1}{m} \sup_{B, ||B||_F\le 1}\<U\Sigma^{-1}B^T,(M-\hM)\>\\
&\le& \frac{1}{m} \sup_{B, ||B||_F\le 1} ||U\Sigma^{-1}B^T||_F\, ||M-\hM||_F\, .\label{eq:IneqNorm2}
\end{eqnarray}
On the other hand
\begin{eqnarray*}
 ||U\Sigma^{-1}B^T||_F^2 = \Trace(B\Sigma^{-1}U^TU\Sigma^{-1}B^T)
= m\Trace(B^TB\Sigma^{-2})\le m\Sigma_{\rm min}^{-2}||B||_F^2\, ,
\end{eqnarray*}
whereby the last inequality follows from the fact that $\Sigma$ is diagonal.
Together (\ref{eq:IneqNorm1}) and (\ref{eq:IneqNorm2}), this implies
the thesis. 
\end{proof}

%
%
\bibliographystyle{amsalpha}

\bibliography{MatrixCompletion}

\end{document}